\documentclass[12pt]{article}

\usepackage{enumitem}
\usepackage{soul} 

\usepackage{graphicx,subfigure,amsmath,amssymb,amsfonts,bm,epsfig,epsf,url,dsfont}

\usepackage{hyperref,cleveref} 
\PassOptionsToPackage{linktocpage}{hyperref}

\usepackage{times}
\usepackage{bbm}      
\usepackage{booktabs}
\usepackage{cases}
\usepackage{fullpage}
\usepackage[small,bf]{caption}
\usepackage{algorithm}
\usepackage{algpseudocode}
\usepackage{mathtools}
\usepackage{thmtools,thm-restate}
\usepackage[top=1in,bottom=1in,left=1in,right=1in]{geometry}

\renewcommand{\phi}{\varphi}

\newcommand{\oL}{\overline{L}}
\newcommand{\diam}{\mathrm{diam}}
\renewcommand{\P}{\mathbb{P}}
\newcommand{\E}{\mathbb{E}}

\newcommand{\R}{\mathbb{R}}

\newcommand{\KL}{\mathrm{KL}}

\newcommand{\cA}{\mathcal{A}}
\newcommand{\cB}{\mathcal{B}}

\newcommand{\cT}{\mathcal{T}}

\newcommand{\Max}{\mathrm{Max}}
\newcommand{\Min}{\mathrm{Min}}
\newcommand{\rmd}{\mathrm{d}}
\def\ds1{\mathds{1}}
\renewcommand{\epsilon}{\varepsilon}
\newcommand{\eps}{\epsilon}

\newcommand{\argmin}{\mathop{\mathrm{argmin}}}

\renewcommand{\tilde}{\widetilde}

\newlength{\minipagewidth}
\newcommand{\beq}{\begin{equation}}
\newcommand{\eeq}{\end{equation}}

\newcommand{\beqa}{\begin{eqnarray}}
\newcommand{\eeqa}{\end{eqnarray}}

\newcommand{\beqan}{\begin{eqnarray*}}
\newcommand{\eeqan}{\end{eqnarray*}}

\def\ba#1\ea{\begin{align*}#1\end{align*}} 
\def\banum#1\eanum{\begin{align}#1\end{align}} 

\def \ent {\mathrm{Ent}}

\def\eps{\varepsilon}
\newtheorem{theorem}{Theorem}
\newtheorem{corollary}{Corollary}

\newtheorem{lemma}{Lemma}

\newtheorem{remark}{Remark}
\newtheorem{proposition}{Proposition}

\newtheorem{definition}{Definition}
\newcommand{\BlackBox}{\rule{1.5ex}{1.5ex}}  
\newenvironment{proof}{\par\noindent{\bf Proof\ }}{\hfill\BlackBox\\[2mm]}

\begin{document}

\title{First-Order Bayesian Regret Analysis of Thompson Sampling}

\author{S\'ebastien Bubeck \\
Microsoft Research
\and Mark Sellke \thanks{This work was done while M. Sellke was an intern at Microsoft Research.} \\
Stanford University}
\date{}
\maketitle

\begin{abstract}
We address online combinatorial optimization when the player has a prior over the adversary's sequence of losses. In this setting, Russo and Van Roy proposed an information theoretic analysis of Thompson Sampling based on the {\em information ratio}, allowing for elegant proofs of Bayesian regret bounds. In this paper we introduce three novel ideas to this line of work. First we propose a new quantity, the {\em scale-sensitive information ratio}, which allows us to obtain more refined {\em first-order regret bounds} (i.e., bounds of the form $O(\sqrt{L^*})$ where $L^*$ is the loss of the best combinatorial action). Second we replace the entropy over combinatorial actions by a {\em coordinate entropy}, which allows us to obtain the first optimal worst-case bound for Thompson Sampling in the combinatorial setting. We additionally introduce a novel link between Bayesian agents and frequentist confidence intervals. Combining these ideas we show that the classical multi-armed bandit first-order regret bound $\tilde{O}(\sqrt{d L^*})$ still holds true in the more challenging and more general semi-bandit scenario. This latter result improves the previous state of the art bound $\tilde{O}(\sqrt{(d+m^3)L^*})$ by Lykouris, Sridharan and Tardos.

Moreover we sharpen these results with two technical ingredients. The first leverages a recent insight of Zimmert and Lattimore to replace Shannon entropy with more refined potential functions in the analysis. The second is a \emph{Thresholded} Thompson sampling algorithm, which slightly modifies the original algorithm by never playing low-probability actions. This thresholding results in fully $T$-independent regret bounds when $L^*\leq \oL^*$ is almost surely upper-bounded, which we show does not hold for ordinary Thompson sampling. 

\end{abstract}

\newpage\tableofcontents\newpage

\section{Introduction}
We first recall the general setting of online combinatorial optimization with both full feedback (full information game) and limited feedback (semi-bandit game). Let $\cA \subset \{0,1\}^d$ be a fixed set of {\em combinatorial actions}, and assume that $m = \|a\|_1$ for all $a \in \cA$. An (oblivious) adversary selects a sequence $\ell_1, \hdots, \ell_T \in [0,1]^d$ of linear functions, without revealing it to the player. At each time step $t=1, \hdots, T$, the player selects an action $a_t \in \cA$, and suffers the instantaneous loss $\langle\ell_t, a_t\rangle$. The following feedback on the loss function $\ell_t$ is then obtained: in the full information game the entire loss vector $\ell_t$ is observed, and in the semi-bandit game only the loss on active coordinates is observed (i.e., one observes $\ell_t \odot a_t$ where $\odot$ denotes the entrywise product). Importantly the player has access to external randomness, and can select their action $a_t$ based on the observed feedback so far. The player's objective is to minimize their total expected loss $L_T = \E \left[ \sum_{t=1}^T \langle\ell_t, a_t\rangle \right]$. The player's perfomance at the end of the game is measured through the {\em regret} $R_T$, which is the difference between the achieved cumulative loss $L_T$ and the best one could have done with a fixed action. That is, with $L^* = \min_{a \in \cA} \sum_{t=1}^T \langle\ell_t, a\rangle$, one has $R_T = L_T - L^*$. The optimal worst-case regret ($\sup_{\ell_1, \hdots, \ell_T \in [0,1]^d} R_T$) is known for both the full information and semi-bandit game. It is respectively of order $m \sqrt{T}$ (\cite{KWK10}) and $\sqrt{m d T}$ (\cite{ABL14}). 

\subsection{First-order regret bounds}
It is natural to hope for strategies with regret $R_T = o(L^*)$. If this holds, one can then claim that $L_T = (1+o(1)) L^*$ (in other words the player's performance is close to the optimal in-hindsight performance up to a smaller order term). However, worst-case bounds may fail to capture this behavior when $L^* \ll T$. The concept of {\em first-order regret bound} tries to remedy this issue, by asking for regret bounds scaling with $L^*$ instead of $T$. In \cite{KWK10} an optimal version of such a bound is obtained for the full information game:

\begin{theorem}[\cite{KWK10}] \label{thm:advfullinfo}
In the full information game, there exists an algorithm such that for any loss sequence one has $R_T = \tilde{O}(\sqrt{m L^*})$.
\end{theorem}

By $\tilde O(\cdot)$ we suppress logarithmic terms, even $\log(T)$. However all our bounds stated in the main body state explicitly the logarithmic dependency.

The state of the art for first-order regret bounds in the semi-bandit game is more complicated. It is known since \cite{AAGO06} that for $m=1$ (i.e., the famous multi-armed bandit game) one can have an algorithm with regret $R_T = \tilde{O}(\sqrt{d L^*})$. On the other hand for $m>1$ the best bound due to \cite{LST18} is $\tilde{O}(\sqrt{(d+m^3) L^*})$. 
Using mirror descent and an entropic regularizer as in \cite{ABL14}, the following bound can be shown:

\begin{theorem} \label{thm:advsemibandit}
In the semi-bandit game, there exists an algorithm such that for any loss sequence one has $R_T = \tilde{O}(\sqrt{dL^*})$.
\end{theorem}

This bound is tight for $L^*=\Theta(mT)$ since the minimax regret for the semi-bandit problem is $\tilde\Theta(\sqrt{mdT})$ (\cite{ABL14}). We derive a version of this result using the recipe first proposed (in the context of partial feedback) in \cite{BDKP15}. Namely, to show the existence of a randomized strategy with regret bounded by $B_T$ for any loss sequence, it is sufficient to show that for any {\em distribution} over loss sequences there exists a strategy with regret bounded by $B_T$ in expectation. Indeed, this equivalence is a simple consequence of the Sion minimax theorem \cite{BDKP15}. In other words to prove Theorem \ref{thm:advsemibandit} it is sufficient to restrict our attention to the {\em Bayesian scenario}, where one is given a prior distribution $\nu$ over the loss sequence $(\ell_1, \hdots, \ell_T) \in [0,1]^{[d]\times [T]}$ and aims for small expected regret with respect to that prior. Importantly note that there is no independence whatsoever in such a random loss sequence, either across times or across coordinates for a fixed time. Rather, the prior is completely arbitrary over the $Td$ different values $(\ell_t(i))_{t\in [T],i\in [d]}$.
\newline

The rest of the paper is dedicated to the (first-order) regret analysis of a particular Bayesian strategy, the famous Thompson Sampling (\cite{Tho33}). In particular we will show that Thompson Sampling implies Theorem \ref{thm:advfullinfo} and an alternate version of Theorem \ref{thm:advsemibandit}. 

\subsection{Thompson Sampling}
In the Bayesian setting one has access to a {\em prior distribution} on the optimal action 
\[
	a^* = \argmin_{a \in \cA} \sum_{t=1}^T  \langle\ell_t, a\rangle.
\] In particular, one can update this distribution as more observations on the loss sequence are collected. More precisely, denote $p_t$ for the posterior distribution of $a^*$ given all the information at the beginning of round $t$ (i.e., in the full information this is $\ell_1, \hdots, \ell_{t-1}$ while in semi-bandit it is $\ell_1 \odot a_1, \hdots, \ell_{t-1} \odot a_{t-1}$). Then Thompson Sampling simply plays an action $a_t$ at random from $p_t$. 

This strategy has recently regained interest, as it is both efficient and successful in practice for simple priors (\cite{empiricalTS}) and particularly elegant in theory. A breakthrough in the understanding of Thompson Sampling's regret was made in \cite{RR14} where an information theoretic analysis was proposed. They consider in particular the combinatorial setting for which they prove the following result:

\begin{theorem}[\cite{RR14}] \label{thm:RR14}
Suppose that under the prior $\nu$, the sequence $(\ell_1, \hdots, \ell_T)$ is i.i.d. Then in the full information game Thompson Sampling satisfies $\E [ R_T ] = \tilde{O}(m^{3/2} \sqrt{T})$, and in the semi-bandit game it satisfies $\E [ R_T ] = \tilde{O}(m \sqrt{d T})$.

Suppose furthermore that under the prior $\nu$, for any $t$, conditionally on $\ell_1, \hdots, \ell_{t-1}$ one has that $\ell_t(1), \hdots, \ell_t(d)$ are independent. Then Thompson Sampling satisfies respectively $\E^{\nu} [ R_T ] = \tilde{O}(m \sqrt{T})$ and $\E^{\nu} [ R_T ] = \tilde{O}(\sqrt{m d T})$ in the full information and semi-bandit game.
\end{theorem}
It was observed in \cite{BDKP15} that the assumption of independence across times is immaterial in the information theoretic analysis of Russo and Van Roy. However it turns out that the independence across coordinates (conditionally on the history) in Theorem \ref{thm:RR14} is key to obtain the worst-case optimal bounds $m\sqrt{T}$ and $\sqrt{m d T}$. One of the contributions of our work is to show how to appropriately modify the notion of entropy to remove this assumption. 
\\

Most importantly, we propose a new analysis of Thompson Sampling that allows us to prove {\em first-order regret bounds}. In various forms we show the following result:
\begin{theorem} \label{thm:TS}
For any prior $\nu$, Thompson Sampling satisfies in the full information game $\E^{\nu} [ R_T ] = \tilde{O}(\sqrt{m \E [ L^* ]})$. Furthermore in the semi-bandit game, $\E^{\nu} [ R_T ] = \tilde{O}(\sqrt{d \E[L^*]})$.
\end{theorem} 
To the best of our knowledge such guarantees were not known for Thompson Sampling even in the full-information case with $m=1$ (the so-called expert setting of \cite{cesa1997use}). Our analysis can be combined with recent work in \cite{TSandMD} which allows for improved estimates based on using mirror maps besides the Shannon entropy.

The link between Theorems~\ref{thm:TS} and \ref{thm:advsemibandit} requires some explanation. In order to recover the full strength of Theorem~\ref{thm:advsemibandit} via the minimax strategy, one would need a regret bound $\tilde{O}(\E[\sqrt{d L^*}])$ which is stronger than the guarantee of Theorem~\ref{thm:TS}. However if an almost sure upper bound $L^*\leq \oL^*$ is known, then Theorem~\ref{thm:TS} implies the existence of a frequentist algorithm attaining regret
\[
	\E^{\nu} [ R_T ] = \tilde{O}\left(\sqrt{d \E[\oL^*]}\right).
\]
In fact the estimate in Theorem~\ref{thm:TS} can be made fully independent of $T$, e.g. with no hidden $\log(T)$ terms. As explained in Section~\ref{sec:threshold}, this is accomplished by a modified Thresholded Thompson sampling algorithm which always avoids low-probability actions. Therefore a frequentist algorithm obtaining the same guarantee exists.

Finally, we note that Thompson sampling against certain artificial prior distributions is also known to obey frequentist regret bounds in the stochastic case (\cite{AG12,thodoris2019graph}). However we emphasize that in this paper, Thompson Sampling assumes access to the true prior distribution for the loss sequence and the guarantees are for expected Bayesian regret with respect to that prior.

\section{Information ratio and scale-sensitive information ratio}
As a warm-up, and to showcase one of our key contributions, we focus here on the full information case with $m=1$ (i.e., the expert setting). We start by recalling the general setting of Russo and Van Roy's analysis (Subsection \ref{sec:RRanalysis}), and how it applies in this expert setting (Subsection \ref{sec:Pinsker}). We then introduce a new quantity, the scale-sensitive information ratio, and show that it naturally implies a first-order regret bound (Subsection \ref{sec:scalesensitive}). We conclude this section by showing a new bound between two classical distances on distributions (essentially the chi-squared and the relative entropy), and we explain how to apply it to control the scale-sensitive information ratio (Subsection \ref{sec:newineq}).

\subsection{Preparation} \label{sec:RRanalysis}
Let us denote $X_t \in \R^d$ for the feedback received at the end of round $t$. That is in full information one has $X_t = \ell_t$, while in semi-bandit one has $X_t = \ell_t \odot a_t$. Let us denote by $\P_t$ the posterior distribution of $\ell_1, \hdots, \ell_{T}$ conditionally on $a_1, X_1, \hdots, a_{t-1}, X_{t-1}$. We write $\E_t$ for the expectation with respect to $\P_t$, which returns a random variable measurable with respect to the sigma algebra generated by $(a_1, X_1, \hdots, a_{t-1}, X_{t-1})$.
In Thompson sampling, we take $a_t \sim p_t$ conditionally on $(a_1, X_1, \hdots, a_{t-1}, X_{t-1})$, where again $p_t$ is the distribution of $a^*$ under $\P_t$. Hence $\E_t[a_t]=p_t$ when viewed as vectors in $\mathbb R^d$. Let $IG_t$ be the mutual information {\em under the posterior distribution} $\P_t$, (denoted in general $I_t$) between $a^*$ and $X_t$, i.e.
\[
	IG_t = I_t(a^*,X_t)= H(p_t) - \E_t[H(p_{t+1})].
\]
(The abbreviation $IG$ stands for ``information gain'' as it represents the amount of new information about the unknown $a^*$.) Let 
\[
	r_t = \E_t [ \langle\ell_t, a_t - a^*\rangle ]
\]
be the instantaneous regret at time $t$. The information ratio introduced by Russo and Van Roy is defined as:
\begin{equation} \label{eq:IR}
\Gamma_t := \frac{r_t^2}{IG_t} \,.
\end{equation}
The point of the information ratio is the following result:
\begin{proposition}[Proposition 1, \cite{RR14}] 
\label{prop:1}
Let $\Gamma>0$ be a positive constant and consider a strategy such that $\Gamma_t \leq \Gamma$ for all $t$ almost surely. Then one has
\[
\E [ R_T ] \leq \sqrt{T \cdot \Gamma \cdot H(p_1)} \,,
\]
where $H(p_1)$ denotes the Shannon entropy of the prior distribution $p_1$ (in particular $H(p_1) \leq \log(d)$).
\end{proposition}
\begin{proof}
The main calculation is as follows:
\begin{equation} 
\label{eq:regretwithIR}
\E [ R_T ] 
= 
\E \left[ \sum_{t=1}^T r_t \right] 
\leq 
\sqrt{T \cdot \E \left[ \sum_{t=1}^T r_t^2 \right]}\leq 
\sqrt{T \cdot \Gamma \cdot \E \left[ \sum_{t=1}^T IG_t \right]} \,.
\end{equation}
Moreover the total information accumulation $\E \left[ \sum_{t=1}^T IG_t \right]$ can be easily bounded via
\begin{equation}
\label{eq:entropy-sum}
\begin{aligned}
	\E \left[ \sum_{t=1}^T IG_t \right] 
	&=
	\E \left[ \sum_{t=1}^T H(p_t)-H(p_{t+1}) \right]
	\\
	&=
	\E[H(p_1)-H(p_{T+1})]
	\\
	&\leq 
	H(p_1).
\end{aligned}
\end{equation}
Substituting into \eqref{eq:regretwithIR} concludes the proof.
\end{proof}

\subsection{Pinsker's inequality and Thompson Sampling's information ratio} \label{sec:Pinsker}
We now describe how to control the information ratio \eqref{eq:IR} of Thompson Sampling in the expert setting. Let
\begin{equation}
\label{eq:ent}
	\ent(p,q) = \sum_{i=1}^d p(i) \log(p(i) / q(i))
\end{equation}
denote the relative entropy. Using the martingale property $\E_t[p_{t+1}]=p_t$ implies
\begin{equation}
\label{eq:entropy-gain}
\begin{aligned}
	\E_t [ \ent(p_{t+1} , p_t) ]
	&=
	\E_t\left[\sum_{i=1}^d p_{t+1}(i) \log(p_{t+1}(i) / p_t(i))\right]
	\\
	&=
	\E_t\left[\sum_{i=1}^d p_{t+1}(i) \log p_{t+1}(i)\right]-\sum_{i=1}^d p_{t}(i) \log p_{t}(i)
	\\
	&=
	H(p_t)-\E_t[H(p_{t+1})]
	\\
	&=
	IG_t.
\end{aligned}
\end{equation}
We also recall Pinsker's inequality:
\begin{equation}
\label{eq:pinsker}
	\|p-q\|_1^2 \leq 2 \cdot \ent(p,q).
\end{equation}
(Here on the left side we view $p$ and $q$ as vectors in $\mathbb R^d$.)

Having completed our preparations we turn to bounding the information ratio. Observe that the posterior distribution $p_t$ of $a^* \in \{e_1,\hdots, e_d\}$ satisfies (again viewing $p_t$ as a vector in $\R^d$): $p_t = \E_t [ a^* ]$. Using the tower rule $\E_t[\E_{t+1}[X]]=\E_t[X]$ for conditional expectations in the second step, we have the important calculation
\begin{equation} 
\label{eq:lostscale}
\begin{aligned}
	r_t 
	&= 
	\E_t [ \langle\ell_t ,a_t - a^*\rangle ] 
	\\
	&=
	\E_t \left[ \E_{t+1}[ \langle\ell_t ,a_t - a^*\rangle ]\right] 
	\\
	&= \E_t \left[\langle\ell_t, \E_{t+1}[ (a_t - a^*) ]\rangle\right]
	\\
	&=
	\E_t [ \langle\ell_t ,p_t - p_{t+1}\rangle ]. 
\end{aligned}
\end{equation}
Here the third step holds because $\ell_t$ is known at time $t+1$ (and note that all steps are really equalities!). Finally we estimate the right hand side above via
\begin{equation}
\label{eq:holder}
	\langle\ell_t ,p_t - p_{t+1}\rangle
	\leq 
	\frac{1}{2}\|p_t - p_{t+1}\|_1
\end{equation}
using the observation $\|\ell_t-(\frac{1}{2},\frac{1}{2},\dots,\frac{1}{2})\|_{\infty} \leq \frac{1}{2}$ (and the fact that $p_t$ and $p_{t+1}$ have the same sum-of-coordinates).
Combining \eqref{eq:lostscale} and \eqref{eq:holder} with Jensen's inequality and \eqref{eq:pinsker} in the first step below and then using \eqref{eq:entropy-gain} yields:
\[
	r_t^2 \leq \frac{1}{2}\cdot \E_t [ \ent(p_{t+1} , p_t)]\stackrel{\eqref{eq:entropy-gain}}{=}\frac{I_t}{2}.
\]
We have shown:
\begin{lemma}[\cite{RR14}]\label{lem:TSIR}
In the expert setting, Thompson Samping's information ratio \eqref{eq:IR} satisfies $\Gamma_t \leq \frac{1}{2}$ for all $t$.
\end{lemma}
Using Lemma \ref{lem:TSIR} in Proposition \ref{prop:1} one obtains the following worst case optimal regret bound for Thompson Sampling in the expert setting:
\[
\E [ R_T ] \leq \sqrt{\frac{T \log(d)}{2}} \,.
\]

\subsection{Scale-sensitive information ratio} \label{sec:scalesensitive}
The information ratio \eqref{eq:IR} was designed to derive $\sqrt{T}$-type bounds (see Proposition \ref{prop:1}).
To obtain $\sqrt{L^*}$-type regret we propose the following quantity which we coin the {\em scale-sensitive information ratio}:
\begin{equation} \label{eq:SSIR}
	\Lambda_t := \frac{(r_t^+)^2}{IG_t \cdot \E_t[ \langle\ell_t , a_t\rangle ]} \,,
\end{equation}
where 
\[
	r_t^+ := \E_t [ \langle\ell_t , \max(0,p_t - p_{t+1})\rangle ].
\]
With this new quantity we obtain the following refinement of Proposition \ref{prop:1}:
\begin{proposition} 
\label{newprop1}
Let $\Lambda>0$ be a positive constant and consider a strategy such that $\Lambda_t \leq \Lambda$ for all $t$ almost surely. Then one has
\[
	\E [ R_T ] \leq \sqrt{\E[L^*] \cdot \Lambda \cdot H(p_1)} + \Lambda \cdot H(p_1) \,.
\]
\end{proposition}
\begin{proof}
The main calculation is as follows:
\begin{eqnarray*} \label{eq:regretwithscaleIR}
\E [ R_T ] \leq \E \left[ \sum_{t=1}^T r_t^+ \right] & \leq & 
\sqrt{\E \left[ \sum_{t=1}^T \E_t[ \langle\ell_t , a_t\rangle ] \right] \cdot \E\left[ \sum_{t=1}^T \frac{(r_t^+)^2}{\E_t[ \langle\ell_t , a_t\rangle]} \right]} \\
& \leq & 
\sqrt{\E[L_T] \cdot \Lambda \cdot \E \left[ \sum_{t=1}^T IG_t \right]} \\
& \stackrel{\eqref{eq:entropy-sum}}{\leq} & 
\sqrt{\E[L_T] \cdot \Lambda \cdot H(p_1)} \,.
\end{eqnarray*}
The proof is concluded from Lemma~\ref{lem:self-bound} just below, with $(a,b,c)=(\mathbb E[L_T],\mathbb E[L^*],\Lambda\cdot H(p_1))$. 
\end{proof}

\begin{lemma}
\label{lem:self-bound}
Suppose $a,b,c\geq 0$ satisfy $a - b \leq \sqrt{a c}$. Then $a - b \leq \sqrt{b c} + c$.
\end{lemma}

\begin{proof}
We asume $a\geq b+c$ as otherwise the result follows immediately. Then
\begin{align*}
	c&\leq \sqrt{ac}
	\\
	\implies 
	a-b+c&\leq 2\sqrt{ac}
	\\
	\implies 
	(\sqrt{a}-\sqrt{c})^2&\leq b
	\\
	\implies
	\sqrt{ac}-c&\leq \sqrt{bc}
	\\
	\implies 
	a-b-c&\leq \sqrt{bc}.
\end{align*}
Here the first implication comes from the main hypothesis and the second from rearranging. The third implication follows by taking the square root of the previous line (both sides are positive since $a\geq b+c$) and multiplying by $\sqrt{c}$. The final implication follows by using again the main hypothesis.
\end{proof}

\subsection{Reversed chi-squared/relative entropy inequality} \label{sec:newineq}
We now describe how to control the scale-sensitive information ratio \eqref{eq:SSIR} of Thompson Sampling in the expert setting. As we saw in Subsection \ref{sec:Pinsker}, the two key inequalites in the Russo-Van Roy information ratio analysis are a simple Cauchy--Schwarz followed by Pinsker's inequality (recall \eqref{eq:lostscale}):
\[
r_t = \E_t[ \langle\ell_t, p_t - p_{t+1}\rangle ] \leq \E_t[ \|\ell_t\|_{\infty} \cdot \|p_t - p_{t+1}\|_1 ] \leq \sqrt{\E_t[ \ent(p_{t+1} , p_t) ] } = \sqrt{IG_t} \,.
\]
In particular, as far as first-order regret bounds are concerned, the ``scale" of the loss $\ell_t$ is lost in the first Cauchy--Schwarz. To control the scale-sensitive information ratio we propose to do the Cauchy--Schwarz step differently and as follows (using the fact that $\ell_t(i)^2\leq \ell_t(i)$):
\begin{eqnarray} 
\label{eq:betterCS} 
	r_t = \E_t[ \langle\ell_t, p_t - p_{t+1}\rangle ] 
	& \leq & 
	\sqrt{\E_t \left[ \sum_{i=1}^d \ell_t(i) p_t(i) \right] \cdot \E_t \left[ \sum_{i=1}^d \frac{(p_t(i) - p_{t+1}(i))^2}{p_t(i)} \right]} 
	\\
	& = & 
	\sqrt{\E_t[\langle\ell_t, p_t \rangle] \cdot \E_t[\chi^2(p_t, p_{t+1})]} \,, \notag
\end{eqnarray}
where $\chi^2(p,q) = \sum_{i =1}^d \frac{(p(i) - q(i))^2}{p(i)}$ is the chi-squared divergence. Thus, to control the scale-sensitive information ratio \eqref{eq:SSIR}, it only remains to relate the chi-squared divergence to the relative entropy.
Unfortunately it is well-known that in general one only has $\ent(q,p) \leq \chi^2(p, q)$ (which is the opposite of the inequality we need). Somewhat surprisingly we show that the reverse inequality in fact holds up to a factor of two true for a slightly weaker form of the chi-squared divergence, which turns out to be sufficient for our needs:
\begin{lemma} 
\label{lem:reversechisquared}
For $p, q \in \R^d_+$ define the {\em positive chi-squared divergence} $\chi^2_+$ by
\[
\chi^2_+(p,q) = \sum_{i : p(i) \geq q(i)} \frac{(p(i) - q(i))^2}{p(i)} \,.
\]
Then one has
\[
\chi^2_+(p,q) \leq 2\cdot\ent(q,p) \,.
\]
\end{lemma}

\begin{proof}
Consider the function $f_t(s) = s \log(s / t) - s + t$, and observe that $f_t''(s) = 1/s$. In particular $f_t$ is convex, and for $s \leq t$ it is $\frac{1}{t}$-strongly convex. Moreover one has $f_t(t)=f_t'(t) = 0$. This directly implies:
\[
f_t(s) \geq \frac{1}{2 t} (t - s)_+^2.
\]
Writing
\[
	\ent(q,p) = \sum_{i=1}^d \left( q(i) \log(q(i) / p(i)) - q(i) + p(i) \right)
\]
and using the above estimate for each $i\in [d]$ concludes the proof.
\end{proof}

We can therefore redo the calculuation \eqref{eq:betterCS} using $r_t^+$ and then invoke Lemma~\ref{lem:reversechisquared} (together with the identity \eqref{eq:entropy-gain}) in the final step:
\begin{equation}
\label{eq:TSSSIR}
\begin{aligned}
	(r_t)_+^2
	&=
	\E_t[\langle \ell_t,(p_t-p_{t+1})_+\rangle]^2
	\\
	&\leq
	\E_t\left[\sum_{i=1}^d \ell_t(i)p_t(i)\right]
	\cdot
	\E_t\left[\sum_{i : p_t(i) \geq p_{t+1}(i)} \frac{(p_t(i) - p_{t+1}(i))^2}{p_t(i)}\right]
	\\
	&=
	\E[\langle \ell_t,p_t\rangle]\cdot \E_t[\chi^2_+(p_t,p_{t+1})]
	\\
	&\leq
	2\cdot \E[\langle \ell_t,p_t\rangle]\cdot IG_t.
\end{aligned}
\end{equation}
Here in the first line, the positive part operation $(\cdot)_+$ is applied entry-wise to $(p_t-p_{t+1})$. We have shown the following.
\begin{lemma} \label{lem:TSSSIR}
In the expert setting, Thompson Samping's scale-sensitive information ratio \eqref{eq:SSIR} satisfies $\Lambda_t \leq 2$ for all $t$.
\end{lemma}
Using Lemma \ref{lem:TSSSIR} in Proposition \ref{newprop1} we arrive at the following new regret bound for Thompson Sampling:
\begin{theorem}
In the expert setting Thompson Sampling satisfies for any prior distribution:
\[
\E[R_T] \leq \sqrt{2\E[L^*] \cdot H(p_1)} + 2H(p_1) \,.
\]
\end{theorem}

\section{Combinatorial setting and coordinate entropy}\label{sec:combinatorial}
We now return to the general combinatorial setting, where the action set $\cA$ is a subset of $\{A \in \{0,1\}^d : \|A\|_1 = m\}$, and we continue to focus on the full information game. Recall that, as described in Theorem \ref{thm:RR14}, Russo and Van Roy's analysis yields in this case the suboptimal regret bound $\tilde{O}(m^{3/2} \sqrt{T})$ (the optimal bound is $m \sqrt{T}$). We first argue that this suboptimal bound comes from basing the analysis on the standard Shannon entropy. We then propose a different analysis based on the {\em coordinate entropy}.

\subsection{Inadequacy of the Shannon entropy}
Let us consider the simple scenario where $\cA$ is the set of indicator vectors for the sets $a_k = \left\{1+ (k-1) \cdot m, \hdots, k \cdot m \right\}$, $k \in [d/m]$. In other words, the action set consists of $\frac{d}{m}$ disjoint intervals of size $m$. This problem is equivalent to a classical expert setting with $d/m$ actions, and losses with values in $[0,m]$. In particular there exists a prior distribution such that any algorithm must suffer regret $m \sqrt{T \log(d/m)} \geq m \sqrt{T H(p_1)}$ (the lower bound comes from the fact that there are only $d/m$ available actions).

Thus we see that, unless the regret bound reflects some of the structure of the action set $\cA \subset \{0,1\}^d$ (besides the fact that elements have $m$ non-zero coordinates), one cannot hope for a better regret than $m \sqrt{T H(p_1)}$. For larger action sets, $H(p_1)$ could be as large as $m\log(d/m)$. Thus, if we are to obtain a regret bound depending only on $m$ and $T$ via the entropy of the optimal action set, the best possible bound will be $m^{3/2} \sqrt{T}$. However the optimal rate for this online learning problem is known to be $\tilde O(m\sqrt{T})$. This suggests that the Shannon entropy is not the right measure of uncertainty in this combinatorial setting, at least if we expect Thompson Sampling to perform optimally.  

Interestingly a similar observation was made in \cite{ABL14} where it was shown that the regret for the standard multiplicative weights algorithm is also lower bounded by the suboptimal rate $m^{3/2} \sqrt{T}$. The connection to the present situation is that standard multiplicative weights corresponds to mirror descent with the Shannon entropy. To obtain an optimal algorithm, \cite{KWK10, ABL14} proposed to use mirror descent with a certain {\em coordinate entropy}. We show next that basing the analysis of Thompson Sampling on this coordinate entropy allows us to prove optimal guarantees.

\subsection{Coordinate entropy analysis}

For any vector $v=(v_1,v_2,\dots,v_d)\in [0,1]^d$, we define its \emph{coordinate entropy} $H^c(v)$ to simply be the sum of the entropies of the individual coordinates:

\[H^c(v)=\sum_{i=1}^d H(v_i)=-\sum_{i=1}^d v_i\log(v_i)+(1-v_i)\log(1-v_i).\]

For a $\{0,1\}^d$-valued random variable such as $a^*$, we define $H^c(a^*)=H^c(\E[a^*])$. Equivalently, the coordinate entropy $H^c(a^*)$ is the sum of the (ordinary) entropies of the $d$ Bernoulli random variables $1_{i\in a^*}$. 

This definition allows us to consider the information gain in each event $[i\in a^*]$ separately in the information theoretic analysis via 
\[
	IG_t^c=H_t^c(p_t)-\E_t[H_t^c(p_{t+1})],
\]
denoting now $p_t = \E_t [ a_t ]$. We define for $p,q\in [0,1]^d$ with $\sum_{i=1}^d p(i)=\sum_{i=1}^d q(i)$:
\begin{equation}
\label{eq:coordinate-entropy}
	\ent^c(p , q)
	=
	\sum_{i=1}^d p(i)\log \frac{p(i)}{q(i)} + (1-p(i))\log\frac{1-p(i)}{1-q(i)}.
\end{equation}
For intuition, note that each term is the relative entropy between Bernoulli variables with means $p(i)$ and $q(i)$, and the above definitions are additive across coordinates. 
Similarly to \eqref{eq:entropy-gain}, we have
\begin{equation}
\label{eq:coordinate-entropy-gain}
\begin{aligned}
	\E_t [ \ent^c(p_{t+1} , p_t) ]
	&=
	\E_t\left[\sum_{i=1}^d p_{t+1}(i) \log \frac{p_{t+1}(i) }{ p_t(i)}\right]
	+
	\E_t\left[\sum_{i=1}^d (1-p_{t+1}(i)) \log \frac{1-p_{t+1}(i) }{ 1-p_t(i)}\right]
	\\
	&=
	\E_t\left[\sum_{i=1}^d p_{t+1}(i) \log p_{t+1}(i) + (1-p_{t+1}(i))\log(1-p_{t+1}(i))\right]
	\\
	&\quad\quad
	-
	\left[
	\sum_{i=1}^d p_{t}(i) \log p_{t}(i)
	+
	(1-p_{t}(i)) \log (1-p_{t}(i))
	\right]
	\\
	&=
	H^c(p_t)-\E_t[H^c(p_{t+1})]
	\\
	&=
	IG^c_t.
\end{aligned}
\end{equation}
Moreover, Lemma~\ref{lem:reversechisquared} continues to hold with the coordinate entropy:
\begin{equation}
\label{eq:coordinate-reversechisquared}
\begin{aligned}
	\frac{1}{2}\chi^2_+(p_t,p_{t+1})
	&\leq
	\ent(p_{t+1},p_t)
	\\
	&=
	\sum_{i=1}^d p(i) \log(p(i)/q(i))
	\\
	&\leq
	\sum_{i=1}^d p(i) \log(p(i)/q(i))
	+
	\sum_{i=1}^d (1-p(i)) \log\frac{1-p(i)}{1-q(i)}
	\\
	&=
	\ent^c(p_{t+1},p_t).
\end{aligned}
\end{equation}
Here in the second-to-last step we used Jensen's inequality and the fact that $\sum_{i=1}^d p_t(i)=\sum_{i=1}^d p_{t+1}(i)$ (as in the usual proof that KL divergence is non-negative).
Next, following \eqref{eq:TSSSIR}, we estimate
\begin{equation}
\label{eq:coordinate-betterCS} 
\begin{aligned}
	(r_t)_+^2 
	&= 
	\E_t[ \langle\ell_t, (p_t - p_{t+1})_+\rangle ]^2
	\\
	&\leq 
	\E_t 
	\left[ 
		\sum_{i=1}^d \ell_t(i) p_t(i) 
	\right] 
	\cdot 
	\E_t 
	\left[ 
		\sum_{i:p_t(i)\geq p_{t+1}(i)} \frac{(p_t(i) - p_{t+1}(i))^2}{p_t(i)} 
	\right]
	\\
	& = 
	\E_t[\langle\ell_t, p_t \rangle] \cdot \E_t[\chi^2_+(p_t, p_{t+1})]
	\\
	&\leq
	2\cdot \E_t[\langle\ell_t, p_t \rangle] \cdot \E_t[\ent^c(p_{t+1},p_t)]
	\\
	&=
	2\cdot \E_t[\langle\ell_t, p_t \rangle] 
	\cdot IG_t^c.
\end{aligned}
\end{equation}
As a result, the scale-sensitive information ratio with coordinate entropy is 
\[
	\Lambda_t^c:= \frac{(r_t^+)^2)}{IG_t^c\cdot \E_t[ \langle\ell_t , a_t\rangle ]}\leq 2.
\]
By exactly the same argument as in Proposition~\ref{newprop1}, we find
\begin{equation}
\label{eq:coordinate-full-feedback-bound}
	\E[R_T]\leq \sqrt{2\E[L^*]H^c(p_1)}+2H^c(p_1).
\end{equation}
To establish the first half of Theorem~\ref{thm:TS} it remains to upper-bound $H(p_1)$ using a function of $(m,d)$. By Jensen's inequality,  
\[
	H^c(p_1) \leq H^c\left(\frac{m}{d},\frac{m}{d},\dots,\frac{m}{d}\right)=m\log\left(\frac{d}{m}\right)+(d-m)\log\left(\frac{d}{d-m}\right).
\]
Using the inequality $\log(1+x)\leq x$ on the second term we obtain
\[
	H^c(p_1)\leq m\log\left(\frac{d}{m}\right)+m\leq m\log(3d/m).
\]
Substituting into \eqref{eq:coordinate-full-feedback-bound} gives the claimed estimate 
\[
	\E[R_T]
	\leq
	\sqrt{2m\log(3d/m)\E[L^*]}+2m\log(3d/m).
\]
\begin{remark}
	The fact we use the coordinate entropy suggests that it is unnecessary to leverage information from correlations between different arms, and we can essentially treat them as independent. In fact, our proofs for Thompson Sampling apply to any algorithm which observes arm $i$ at time $t$ with probability $p_t(i\in a^*)$. This remark extends to the thresholded variants of Thompson Sampling we discuss at the end of the paper.
\end{remark}

\section{Bandit Setting}\label{sec:bandit}

Now we return to the $m=1$ setting and consider the case of bandit feedback. We again begin by recalling the analysis of Russo and Van Roy, and then adapt it in analogy with the scale-sensitive framework. For most of this section, we require that an almost sure upper bound $L^*\leq \oL^*$ for the loss of the best action is given to the player. Under this assumption we show that Thompson Sampling obtains a regret bound $\tilde O(\sqrt{H(p_1)d\oL^*})$, by using a bandit analog of the method in the previous section. This estimate can be improved with the method of \cite{TSandMD} which shows how to analyze Thompson Sampling based on online stochastic mirror descent. By using a logarithmic regularizer in the analysis, we obtain a regret bound depending only on $\E[L^*]$, i.e. {\em without} the assumption $L^*\leq \oL^*$, matching the statement of Theorem~\ref{thm:TS}.

\subsection{The Russo and Van Roy Analysis for Bandit Feedback}

In the bandit setting we cannot bound the regret by the movement of $p_t$. Indeed, the calculation~\eqref{eq:lostscale} relies on the fact that $\ell_t$ is known at time $t+1$ which is only true for full feedback. However, a different information theoretic calculation gives a good estimate. Below, we set
\[
	\bar\ell_t(i)=\E_t[\ell_t(i)],
	\quad\text{ and }\quad
	\bar\ell_t(i,j)=\E_t[\ell_t(i)|a^*=j].
\]
The analog of \eqref{eq:lostscale} which we take as our starting point follows. For later flexibility we allow algorithms that are not Thompson sampling. 
\begin{proposition}
\label{prop:bandit-regret-basic}
Suppose an algorithm for the bandit game has $p_t(i)=\mathbb P_t[i=a^*]$ and plays from $a_t\sim\hat p_t$. Then the expected regret is given by
\[
	R_T=\sum_{t=1}^T r_t
\]
for
\[
	r_t=\sum_{i=1}^d \big(\hat p_t(i)\bar\ell_t(i)-p_t(i)\bar\ell_t(i,i))\big).
\]
In the case $\hat p_t=p_t$ of Thompson sampling, this formula simplifies to
\[
	r_t=\sum_{i=1}^d \big(p_t(i)(\bar\ell_t(i)-\bar\ell_t(i,i))\big).
\]
\end{proposition}

\begin{proof}
We will claim that $r_t=\E_t[\ell(a_t)-\ell(a^*)]$ which implies the first statement. Indeed, one immediately verifies that 
\begin{align*}
	\E_t[\ell(a_t)]
	&=
	\sum_{i=1}^d \hat p_t(i)\bar\ell_t(i)
	;
	\\
	\E_t[\ell(a^*)]
	&=
	\sum_{i=1}^d p_t(i)\bar\ell_t(i,i).
\end{align*}
\end{proof}

For $x,y\in [0,1]$ we let 
\[
	\ent[x,y]=-x\log(x/y)-(1-x)\log\frac{1-x}{1-y}
\]
denote the binary entropy between the corresponding Bernoulli random variables. Thus $\ent[x,y]=\ent^c[x,y]$ for scalars $x,y\in [0,1]$.

\begin{lemma}
[{\cite{RR14}}]
In the bandit setting, Thompson Sampling's information ratio satisfies $\Gamma_t\leq d$ for all $t$. Therefore it has expected regret $\E[R_T]\leq \sqrt{dTH(p_1)}$.
\end{lemma}

\begin{proof} Using Proposition~\ref{prop:bandit-regret-basic}, Cauchy--Schwarz and finally Pinsker,
\begin{align*}
	r_t 
	&= 
	\sum_{i=1}^d p_t(i)(\bar\ell_t(i)-\bar\ell_t(i,i)) 
	\\
	&\leq 
	\sqrt{d \sum_{i=1}^d p_t(i)^2(\bar\ell_t(i)-\bar\ell_t(i,i))^2}
	\\
	&\leq
	\sqrt{d \sum_{i=1}^d p_t(i)^2 \ent[\bar\ell_t(i,i),\bar\ell_t(i)] }.
\end{align*}
By Lemma~\ref{lem:partialobservationIR} below, this means
\[
	r_t\leq \sqrt{d\cdot IG_t}
\]
which is equivalent to $\Gamma_t\leq d$. 
\end{proof}

The following lemma generalizes a calculation in \cite{RR14}. In it, we take $S\subseteq [d]$ to be a random set of arms. In the bandit setting we will always take $S=\{a^*\}$, but less obvious choices for $S$ will be considered in the semibandit game. (In all our applications $S$ will be a function of $(\ell_t(i))_{(t,i)\in [T]\times [d])}$ but even this assumption is not necessary below.)

We also let $A_t\subseteq [d]$ be the set of actions chosen by the player at time $t$, so $A_t=\{a_t\}$ when $m=1$. It will be convenient to use the notation:
\begin{align*}
	p_t(i\in S)&=\mathbb P[i\in S],
	\\
	\hat p_t(i)&=\mathbb P[i\in A_t],
	\\
	\bar\ell_t(i,i\in S)&=\E[\ell_t(i)|i\in S],
	\\
	IG^c_t(S)&=\sum_{i\in S} IG^c_t(i).
\end{align*}
Throughout the later parts of this paper, we will use various choices of $S$, for instance the top $m$ actions. In the proof below, we also denote by $\mathcal L_t(X)$ the law of the random variable $X$ at time $t$. As mentioned previously we write $I_t[X,Y]$ to denote the mutual information between $X$ and $Y$ conditioned on all observations before time $t$. 

\begin{restatable}{lemma}{partialobservationIR}
\label{lem:partialobservationIR}

Suppose a Bayesian player is playing a semi-bandit game with a random subset $S\subseteq [d]$ of arms. Each round $t$, the player picks some subset $A_t$ of arms and observes the losses $(\ell_t(i))_{i\in A_t}$.
Then
\[
	\sum_{i=1}^d \hat p_t(i)p_t(i\in S) \ent[\bar\ell_t(i,i\in S),\bar\ell_t(i)]\leq IG^c_t[S].
\]


\end{restatable}

\begin{proof} 
Let $\tilde \ell_t(i)$ be a $\{0,1\}$-valued random variable with expected value $\bar\ell_t(i)$ and conditionally independent of everything else.
The data processing inequality gives the inequality
\[
	I_t[\tilde \ell_t(i),1_{i\in S}]
	\leq 
	I_t[\ell_t(i),1_{i\in S}]
\]	
between mutual informations.
We explicitly write out the mutual information on the left-hand side. Things simplify since the random variable $\tilde\ell_t(i)$ is Bernoulli:
\begin{align*}
	I_t[\tilde \ell_t(i),1_{i\in S}]
	&=
	p_t(i\in S)D_{\KL}(\tilde\ell_t(i|i\in S)~||~\tilde\ell_t(i))
	+
	p_t(i\notin S)D_{\KL}(\tilde\ell_t(i|i\notin S)~||~\tilde\ell_t(i))
	\\
	&=
	p_t(i\in S)\ent[\bar\ell_t(i|i\in S),\bar\ell_t(i)]
	+
	p_t(i\notin S)\ent[\bar\ell_t(i|i\notin S),\bar\ell_t(i)]
	\\
	&\geq
	p_t(i\in S)\ent[\bar\ell_t(i|i\in S),\bar\ell_t(i)].
\end{align*}
%
Next we observe that the event $[i\in A_t]$ holds with probability $\hat p_t(i)$ independently of everything else. Therefore
\begin{align*}
	\hat p_t(i)p_t(i\in S) \ent[\bar\ell_t(i,i\in S),\bar\ell_t(i)]
	&\leq 
	\hat p_t(i)I_t[\tilde \ell_t(i),1_{i\in S}]
	\\
	&\leq
	\hat p_t(i)I_t[\ell_t(i),1_{i\in S}]
	\\
	&=
	I_t[\ell_t(i)1_{i\in A_t},1_{i\in S}]
	\\
	&\leq 
	I_t[(A_t,\vec\ell_t(A_t)),1_{i\in S}]
	\\
	&=IG_t[1_{i\in S}].
\end{align*}
Here the last inequality step holds because $(A_t,\vec\ell_t(A_t))$ determines $\ell_t(i)1_{i\in A_t}$. Summing over $i\in [d]$ completes the proof.
\end{proof}

The next lemma is a scale-sensitive analog of an information ratio bound for partial feedback, in the sense that a similar improved Cauchy--Schwarz inequality is used. However going from such a statement to a regret bound turns out to be more involved in the small loss setting, so we do not try to push the analogy too far.

\begin{restatable}{lemma}{scalesensitivebanditIR}
\label{lem:scalesensitivebanditIR}
In the setting of Lemma~\ref{lem:partialobservationIR},
\[
	\sum_{i=1}^d \hat p_t(i)p_t(i\in S)\left(\frac{(\bar\ell_t(i)-\bar\ell_t(i,i\in S))_+^2}{\bar\ell_t(i)}\right) \leq 2\cdot IG_t^c[S].
\]
\end{restatable}

\begin{proof}
By the proof of Lemma~\ref{lem:reversechisquared},
\begin{align*}
	&\sum_{i=1}^d\hat p_t(i) p_t(i\in S)\left(\frac{(\bar\ell_t(i)-\bar\ell_t(i,i\in S))_+^2}{\bar\ell_t(i)}\right)\\
	&\quad\quad\quad
	\leq 2\sum_{i=1}^d \hat p_t(i) p_t(i) 
	\Big(
		\ent\big(
			\bar\ell_t(i,i\in S),\bar\ell_t(i)
		\big)-\bar\ell_t(i,i\in S)+\bar\ell_t(i)
	\Big)
\end{align*}
and 
\begin{align*}
	&\sum_{i=1}^d\hat p_t(i) p_t(i\notin S)\left(\frac{(\bar\ell_t(i)-\bar\ell_t(i,i\notin S))_+^2}{\bar\ell_t(i)}\right)\\
	&\quad\quad\quad
	\leq 2\sum_{i=1}^d \hat p_t(i) p_t(i) 
	\Big(
		\ent\big(
			\bar\ell_t(i,i\notin S),\bar\ell_t(i)
		\big)-\bar\ell_t(i,i\notin S)+\bar\ell_t(i)
	\Big).
\end{align*}
Summing and noting that 
\begin{align*}
	p_t(i\in S)\bar\ell_t(i,i\in S)+p_t(i\notin S)\bar\ell_t(i,i\notin S)
	&=
	p_t(i\in S)\bar\ell_t(i)+p_t(i\notin S)\bar\ell_t(i)
	\\
	&=\bar\ell_t(i),
\end{align*}
we obtain
\[
	\sum_{i=1}^d\hat p_t(i) p_t(i\in S)\left(\frac{(\bar\ell_t(i)-\bar\ell_t(i,i\in S))_+^2}{\bar\ell_t(i)}\right)
	+
	\sum_{i=1}^d\hat p_t(i) p_t(i\notin S)\left(\frac{(\bar\ell_t(i)-\bar\ell_t(i,i\notin S))_+^2}{\bar\ell_t(i)}\right)
\]
\begin{align*}
	&\leq 
	2\sum_{i=1}^d \hat p_t(i)
	\bigg(
		p_t(i\in S)\ent(\bar\ell_t(i,i\in S),\bar\ell_t(i))+p_t(i\notin S)\ent(\bar\ell_t(i,i\notin S),\bar\ell_t(i)) 
	\bigg)
	\\
	&\leq
	2\sum_{i=1}^d \hat p_t(i)
	p_t(i\in S)\ent\big(\bar\ell_t(i,i\in S),\bar\ell_t(i)\big)
	\\
	&\stackrel{Lem~\ref{lem:partialobservationIR}}{\leq}
	2\cdot IG_t^c[S].
\end{align*}
\end{proof}

\subsection{General Theorem on Bayesian Agents}

Here we state a theorem on the behavior of a Bayesian agent in an online learning environment. In the next subsection we use it to give a nearly optimal regret bound for Thompson Sampling with bandit feedback. This theorem is stated in a rather general way to encompass the semi-bandit setting as well as the Thresholded Thompson Sampling discussed later.

As with the rest of this paper, the theorem below concerns the \emph{Bayes-optimal} setting, in which a Bayesian agent starts with a prior and the true environment is generated from that prior. As before, we let $p_t(i)=\mathbb P_t[i\in A^*]$ be the time-$t$ probability that $i$ is one of the top $m$ arms and $\hat p_t(i)=\mathbb P_t[i\in A_t]$ the probability that the player plays arm $i$ in round $t$.

We also suppose that there exist constants $\frac{1}{\oL^*}\leq\gamma_1\leq \gamma_2$ and a time-varying partition 
\begin{equation}
\label{eq:rare-common-partition}
	[d]=\mathcal R_t\cup\mathcal C_t
\end{equation}
of the action set into \emph{rare} and \emph{common} arms such that:
\begin{enumerate}
	\item If $i\in \mathcal C_t$, then $\hat p_t(i), p_t(i)\geq \gamma_1$.
	\item If $i\in \mathcal R_t$, then $\hat p_t(i)\leq p_t(i)\leq \gamma_2$.
\end{enumerate}

The partition $[d]=\mathcal R_t\cup\mathcal C_t$ into arms with low and high probability to be optimal will be used to analyze the original Thompson sampling algorithm, as well as \emph{Thresholded} Thompson Sampling which plays only from $\mathcal C_t$.

\begin{restatable}{theorem}{bayesianexplorer}
\label{thm:bayesianexplorer}

Consider an online learning game with arm set $[d]$ and random sequence of losses $\ell_t(i)$, in the Bayes-optimal setting. Assume there always exists an action with total loss at most $\oL^*$. Each round, the player plays some action $A_t\in \binom{[d]}{m}$, i.e. a set of $m\geq 1$ arms, and pays/observes the loss for each of them. Moreover suppose a partition \eqref{eq:rare-common-partition} exists and the properties above hold for it.
Then the following statements hold for every $i\in[d]$.
\begin{enumerate}[label=\emph{\Alph*})]
	\item The expected loss incurred \textbf{by the player} from arm $i$ while $i\in \mathcal R_t$ is rare is
	\[
		\E\left[\sum_{t\in [T]:\text{ }i\in \mathcal R_t} \hat p_t(i)\ell_t(i)\right]
		\leq 
		2\gamma_2 \oL^*+8\log(T)+4.
	\]
	\item The expected total loss that arm $i$ incurs while $i\in\mathcal C_t$ is common is 
	\[
		\E\left[\sum_{t\in [T]:\text{ }i\in \mathcal C_t} \ell_t(i)\right] \leq \oL^*+2\left(\log\left(\frac{1}{\gamma_1}\right)+10\right)\sqrt{\frac{\oL^*}{\gamma_1}}.
	\]
\end{enumerate}
\end{restatable}

The use of Theorem~\ref{thm:bayesianexplorer} will become clear in the remainder of this section. We give the proof in the Appendix but outline next some of the key ideas.

\subsubsection{Proof Ideas for Theorem~\ref{thm:bayesianexplorer}}

As initial intuition for Theorem~\ref{thm:bayesianexplorer}, recall that for any bandit algorithm satisfying $\hat p_t(i)>0$ for all $(t,i)\in [T]\times [d]$, one may construct the importance-weighted estimate
\[
	\hat L_t(i)=\sum_{s\leq t} \frac{\ell_t(i)1_{i\in A_t}}{\hat p_t(i)}
\]
for $L_t(i)=\sum_{s\leq t}\ell_t(i)$. Moreover this estimate is unbiased in the sense that for all fixed $(t,i)\in [T]\times [d]$ and any fixed loss sequence, we have
\[
	\E[\hat L_t(i)]=L_t(i).
\]
In fact our analysis uses unbiased loss estimates for common arms $i\in \mathcal C_t$, but \textbf{under}biased estimates for $i\in\mathcal R_t$. This is because dividing by $\hat p_t(i)$ leads to a large variance in the natural unbiased estimate when $\hat p_t(i)$ is small. Moreover we separately construct loss estimates for $\mathcal C_t$ and $\mathcal R_t$. The precise definitions are given in the following table.

\begin{table}[ht]
\caption{Notations for unbiased and underbiased loss estimators.}
\label{table:loss-notation}
\begin{center}
\begin{tabular}{ |c|c|c|c|  } 
 \hline
 $\ell_t^{\mathcal R}(i)=\ell_t(i)\cdot 1_{i\in \mathcal R_t}$ 
 & 
 $u_t^{\mathcal R}(i)=\frac{\ell^{\mathcal R}_t(i)\cdot 1_{i\in A_t}}{\gamma_2}$ 
 & 
 $L_t^{\mathcal R}(i)=\sum_{s\leq t}\ell_s^{\mathcal R}(i)$ 
 &  
 $U_t^{\mathcal R}(i)=\sum_{s\leq t}u_s^{\mathcal R}(i)$ \\ 
 \hline 
 $\ell_t^{\mathcal C}(i)=\ell_t(i)\cdot 1_{i\in \mathcal C_t}$ 
 & 
 $u_t^{\mathcal C}(i)=\frac{\ell^{\mathcal C}_t(i)\cdot 1_{i\in A_t}}{\hat p_t(i)}$ 
 & 
 $L_t^{\mathcal C}(i)=\sum_{s\leq t}\ell_s^{\mathcal C}(i)$ 
 & 
 $U_t^{\mathcal C}(i)=\sum_{s\leq t}u_s^{\mathcal C}(i)$  
 \\
 \hline
\end{tabular}
\end{center}
\end{table}

The variables $\ell_t^{\mathcal R}(i)$ and $\ell_t^{\mathcal C}(i)$  are the losses of arm $i$, separated into rare and common contributions. Thus the variables $L_t^{\mathcal R}$ and $L_t^{\mathcal C}$ track the cumulative rare and common loses. Each $u_t^{\mathcal C}(i)$ is an unbiased estimate of $\ell_t^{\mathcal C}(i)$ while $u_t^{\mathcal R}(i)$ is an \textbf{under}biased estimate of $\ell_t^{\mathcal R}(i)$. The same properties carry over for the $U_t$ variables as unbiased or underbiased estimates of the $L_t$.

The central idea behind Theorem~\ref{thm:bayesianexplorer} is that the online player has enough information to compute the loss estimates $U_t^{\mathcal R}(i)$ and $U_t^{\mathcal C}(i)$. For example, suppose that $U_t^{\mathcal C}(i)\gg \oL^*$ is much larger than $\oL^*$. It is easy to show that $U_t^{\mathcal C}(i)$ is provably an accurate estimate for $L_t(i)$ in the frequentist sense (via a martingale generalization of the Chernoff bound). Given this, we might hope the Bayesian player would ``automatically'' infer that the optimality $i\in A^*$ of arm $i$ is extremely unlikely, and hence $i\in\mathcal R_s$ would hold for $s>t$. The Bayes-optimality assumption makes this hope a reality! Indeed the tower rule for conditional expectations implies
\[
	\E[\mathbb P_t[E]]=\mathbb P[E]
\]
for any event $E$. Then roughly speaking, if $\mathbb P[E]\approx 1$, it follows that 
\begin{equation}
\label{eq:tower-rule}
	\mathbb P[\mathbb P_t[E]\approx 1]\approx 1.
\end{equation}
Moreover by Bayes-optimality \textbf{the algorithm plays based on $\mathbb P_t$}. In particular we might take $E$ to be something like ``the error $|U_t^{\mathcal C}(i)-L_t^{\mathcal C}(i)|$ is small''. Then on the event $E$, the observation $U_t^{\mathcal C}(i)\gg \oL^*$ implies that $i\notin A^*$. Therefore \eqref{eq:tower-rule} implies that with high probability, we have
\[
	\mathbb P_t[i\in A^*]\leq 1-\mathbb P_t[E]\approx 1.
\]
Roughly speaking this argument shows that $i\in\mathcal R_t$ must hold with high probability once $U_t^{\mathcal C}(i)\gg \oL^*$, as long as $|U_t^{\mathcal C}(i)-L_t^{\mathcal C}(i)|$ is relatively small with high probability.

In fact since $U_t^{\mathcal C}(i)$ is an unbiased estimator for $L_t(i)$, the approximation error $|U_t^{\mathcal C}(i)-L_t^{\mathcal C}(i)|$ can be shown to be small with high probability when the variance of the estimate is controlled. This holds when the probabilities $\hat p_t(i)\geq \gamma>0$ are uniformly lower-bounded, which holds by construction within $\mathcal C_t$. As a result, the above proof outline works for Theorem~\ref{thm:bayesianexplorer}B.

The proof of Theorem~\ref{thm:bayesianexplorer}A uses a similar technique although the quantity to be bounded is different. It argues that any player-incurred loss from rare arms must quickly make $U_t^{\mathcal R}(i)$ extremely large. Indeed since all rare arms $i\in\mathcal R_t$ have $\hat p_t(i)\leq \gamma_2$, we expect 
\[
	U_t^{\mathcal R}(i)\gg \oL^*
\]
to hold once 
\[
	\sum_{t\in [T]:i\in \mathcal R_t} \hat p_t(i)\ell_t(i)\gg \gamma_2 \oL^*.
\] 
A statement of this form can in fact be shown using a one-sided martingale concentration inequality.
However we take advantage of this conclusion in a different way. Namely we argue that once $U_t^{\mathcal R}\gg \oL^*$ occurs, $\hat p_t(i)$ must become so small that arm $i$ is pulled extremely infrequently. For finite $T$, the slow-down in exploring arm $i$ is so drastic that arm $i$ is only pulled $O(\log T)$ times while $i\in\mathcal R_t$. The $\log(T)$ term in the result is crucial here because we cannot argue that $\hat p_t(i)$ becomes zero but only that it becomes extremely small. Given infinite time, Thompson sampling can potentially return to explore every arm $i$ until paying regret $\oL^*+1$ per arm (at which point $p_t(i)$ finally becomes $0$); see Theorem~\ref{thm:Tdependent} for a concrete example. This issue is circumvented by the Thresholded Thompson sampling algorithm discussed later, which does attain fully $T$-independent small loss regret when $L^*\leq \oL^*$ is known to hold almost surely.

\subsection{First-Order Regret for Bandit Feedback}

As suggested by Theorem~\ref{thm:bayesianexplorer}, we split the action set into \emph{rare} and \emph{common} arms for each round. For the $m=1$ bandit case, we define for some constant $\gamma>0$:
\begin{equation}
\label{eq:rare-common-bandit}
	\mathcal R_t=\{i\in [d]:p_t(i)\leq\gamma\},
	\quad\quad
	\mathcal C_t=\{i\in [d]:p_t(i)>\gamma\}
\end{equation}
Note that an arm $i$ can switch between rare and common over time. As in Table~\ref{table:loss-notation} we split the loss function into 
\[
	\ell_t(i)=\ell_t^{\mathcal R}(i)+\ell_t^{\mathcal C}(i)
\] 
via 
\[
	\ell_t^{\mathcal R}(i)=\ell_t(i)1_{i\in \mathcal R_t},
	\quad
	\text{ and }
	\quad
	\ell_t^{\mathcal C}(i)=\ell_t(i)1_{i\in \mathcal C_t}.
\]
Recalling Proposition~\ref{prop:bandit-regret-basic}, in the bandit case it will be convenient to redefine
\[
	r_t^+=p_t(i)\cdot (\bar\ell_t(i)-\bar\ell_t(i,i))_+.
\]
Now we are ready to prove the first-order regret bound for bandits.

\begin{theorem}
\label{thm:banditregret}

Suppose that $L^*\leq \oL^*$ almost surely. Then Thompson Sampling with bandit feedback obeys the regret estimate 
\[
	\E[R_T]\leq O\left(\sqrt{H(p_1)d\oL^*}+d\log^2(\oL^*)+d\log(T)\right).
\]

\end{theorem}

\begin{proof} Fix $\gamma>0$ and define $\mathcal R_t$ and $\mathcal C_t$ as in \eqref{eq:rare-common-bandit}. We apply Proposition~\ref{prop:bandit-regret-basic} and split off the rare arm losses at the start of the analysis:
\begin{equation}
\label{eq:rare-common-decomposition}
\begin{aligned}
	\E[R_T]
	&\leq 
	\E\left[\sum_{t=1}^T r_t^+\right]
	\\
	&=
	\E\left[\sum_{t=1}^T p_t(i)\cdot (\bar\ell_t(i)-\bar\ell_t(i,i))_+\right] 
	\\
	&\leq 
	\E\left[\sum_{(t,i):i\in \mathcal R_t}p_t(i)\bar\ell_t(i)\right] 
	+ 
	\E\left[\sum_{(t,i):i\in \mathcal C_t}  p_t(i)\cdot (\bar\ell_t(i)-\bar\ell_t(i,i))_+\right].
\end{aligned}
\end{equation}
The first term is bounded by Theorem~\ref{thm:bayesianexplorer}A with the rare/common partition above, $\gamma_1=\gamma_2=\gamma$, and $\hat p_t(i)=p_t(i)$. For the second term, again using Cauchy--Schwarz and then Lemmas~\ref{lem:reversechisquared} and \ref{lem:scalesensitivebanditIR} gives:
\begin{equation}
\label{eq:common-CS-bandit}
\begin{aligned}
	\E\left[\sum_{(t,i):i\in \mathcal C_t}  p_t\cdot (\bar\ell_t(i)-\bar\ell_t(i,i))_+\right] 
	&\leq 
	\sqrt{
		\E\left[
			\sum_{(t,i):i\in \mathcal C_t} 
			\bar\ell_t(i)
		\right]
		\E\left[
			\sum_{(t,i):i\in \mathcal C_t} 
			p_t(i)^2 
			\cdot 
			\frac{(\bar\ell_t(i)-\bar\ell_t(i,i))_+^2}{\bar\ell_t(i)}
		\right]
	}
	\\
	&\leq 
	\sqrt{2\cdot\E\left[\sum_{(t,i):i\in \mathcal C_t} \ell_t(i)\right]\cdot H(p_1)}.
\end{aligned}
\end{equation}
Substituting in the conclusion of Theorem~\ref{thm:bayesianexplorer}B and combining gives:
\[
	\E[R_T]
	\leq 
	d(2\gamma \oL^*+8\log(T)+4)
	+
	\sqrt{H(p_1)d\left(\oL^*+2\left(\log\left(\frac{1}{\gamma}\right)+10\right)\sqrt{\frac{\oL^*}{\gamma}}\right)}.
\]
Taking $\gamma=\min\left(1,\frac{\log^2(\oL^*)}{\oL^*}\right)$ completes the proof.
\end{proof}

\section{Improved Estimates Beyond Shannon Entropy}

In recent work \cite{TSandMD}, it is shown that Thompson sampling can be analyzed using any mirror map, with the same guarantees as online stochastic mirror descent. See also \cite{TShellinger} which improves the Russo and Van Roy entropic bound using Tsallis entropy, and \cite{lattimore2020mirror} which further elucidates the connection between generalized information ratios and mirror descent. Their work is compatible with our methods for first order analysis, allowing for further refinements. By using the Tsallis entropy we remove the $\log(d)$ factor potentially coming from $H(p_1)$ in Theorem~\ref{thm:banditregret}, and also gain the potential for polynomial-in-$d$ savings for informative priors. By using the log barrier we obtain a small loss bound depending only on $\mathbb E[L^*]$ instead of requiring an almost sure upper bound $\oL^*$. 

\begin{definition}

For $\alpha\in (0,1)$, the $\alpha$-Tsallis entropy of a probability vector $p$ is 
\[
	H_{\alpha}(p)=\frac{\left(\sum_{i=1}^d p_i^{\alpha}\right)-1}{\alpha(1-\alpha)}.
\]

\end{definition}

Note that with $d$ actions, $H_{\alpha}(p)\leq \frac{d^{1-\alpha}}{\alpha(1-\alpha)}$. 

\begin{theorem}
\label{thm:MDtsallis}
Suppose that $L^*\leq \oL^*$ almost surely. Then Thompson Sampling with bandit feedback obeys the regret estimate 
\[
	\E[R_T]
	\leq 
	\frac{1}{\sqrt{\alpha(1-\alpha)}}
	O\left(\sqrt{H_{\alpha}(p_1) d^{\alpha}\oL^*}+d\log^2(\oL^*)+d\log(T)\right).
\]
Taking the worst case $H_{\alpha}(p_1)=d^{1-\alpha}$ over $p_1$ yields the regret estimate
\[
	\E[R_T]\leq O_{\alpha}\left(\sqrt{d\oL^*}+d\log^2(\oL^*)+d\log(T)\right).
\]
\end{theorem}

\begin{theorem}

\label{thm:MDlogbarrier}

Thompson Sampling with bandit feedback obeys the regret estimate 
\[
	\E[R_T]=O(\sqrt{d\E[L^*]\log(T)}+d\log(T).
\]
\end{theorem}

We observe that for a highly informative prior, Theorem~\ref{thm:MDtsallis} may be much tighter than a worst case bound. For example if $p_1(i)\lesssim i^{-\beta}$ for some $\beta>1$, then for $\alpha\geq\frac{1}{\beta}$ we will have $H_{\alpha}(p_1)$ bounded independently of $d$. Hence the main term of the regret will be $O_{\alpha}(\sqrt{d^{\alpha}\oL^*})$, meaning the regret bound is improved multiplicatively by a power of $d$.

We also remark that Theorem~\ref{thm:MDlogbarrier} actually does not require Theorem~\ref{thm:bayesianexplorer}. As a result its proof is in the end somewhat shorter than that of Theorem~\ref{thm:MDtsallis}. However the $\oL^*$-dependent results have the interesting advantage of leading to fully $T$-independent regret with Thresholded Thompson Sampling as explained in the next section. We now turn to the proofs which adapt the ideas of \cite{TSandMD} to our setting.

\begin{definition}
\label{defn:admissible}
A $C^3$ function $f: [0,1]\to\mathbb R^+\cup\{\infty\}$ is \textbf{admissible} if for all $x\in [0,1]$,
\begin{enumerate}
	\item $f'(x)\leq 0$.
	\item $f''(x)\geq 0$.
	\item $f'''(x)\leq 0$.
\end{enumerate} 

\end{definition}

For $f$ admissible we consider the potential function
\[
	F(v)=\sum_{i=1}^d f(v),\quad v\in [0,1]^d.
\] 
The admissible functions we will consider are:
\begin{itemize}
	\item $f(x)=x\log(x)$ (negative entropy);
	\item $f(x)=-x^{1/2}$ (negative Tsallis entropy);
	\item $f(x)=-\log(Tx+1)$ (log barrier). 
\end{itemize}
Letting $\Delta_d$ denote the simplex of $d$-dimensional probability vectors, we set
\[
	\Max(F)=\max_{p\in\Delta_d}F(p),
	\quad\quad
	\Min(F)=\min_{p\in\Delta_d}F(p)
\] 
and also 
\[
	\diam(F)=\Max(F)-\Min(F).
\]
Note that convexity of $f$ implies $\Max(F)=F(1)+(d-1)F(0)$ and $\Min(F)=dF(1/d)$.

It will later be convenient to use semibandit analogs of these quantities. Let 
\[
	\Delta_{d,j}=\{x\in [0,1]^d,\sum_{i=1}^d x_i=j\}
\]
and define
\begin{align}
\label{eq:max-j}
	\Max_j(F)&=\max_{p\in\Delta_d}F(p);
	\\
\label{eq:min-j}
	\Min_j(F)&=\min_{p\in\Delta_d}F(p);
	\\
\label{eq:diam-j} 
	\diam_j(F)&=\Max_j(F)-\Min_j(F).
\end{align}

While studying the full-feedback scenario, we crucially used in Lemma~\ref{lem:reversechisquared} a one-sided strong convexity property of the entropy function. Admissibility is the condition required to generalize this calculation. Indeed, for $x,y\in [0,1]$, admissibility implies 
\begin{equation}
\label{eq:admissible-taylor}
	f(y)-f(x)\geq f'(x)(y-x)+\frac{f''(x)}{2}(x-y)_+^2.
\end{equation}
This is because $f(b)$ is convex on $b\geq a$ and $f''(a)$--strongly convex on $b\leq a$.

The proposition below uses Cauchy--Schwarz with scale-sensitive scaling in this general setting. For the sake of later application we work in the general $m\geq 1$ setting. Similarly to before, for some random set $S\subseteq [d]$ of arms, we set 
\begin{align*}
	p_t(i)&=\mathbb P_t[i\in S];
	\\
	\hat p_t(i) &= \mathbb P_t[i\in A_t];
	\\
	\ell_t(i,i)&=\mathbb E_t[\ell_t(i)~|~i\in S].
\end{align*}
Thus for Thompson sampling, $S=\{a^*_1,\dots,a^*_m\}$ and $\hat p_t=p_t$.

\begin{proposition}
\label{prop:TSandMD}
Let $f$ be admissible. Then
\[
	\E_t[F(p_{t+1})-F(p_{t})] \geq \sum_{i=1}^d \hat p_t(i)p_t(i)^2 f''(p_t(i)) \frac{\left(\bar\ell_t(i)-\bar\ell_t(i,i)\right)_+^2}{2\bar\ell_t(i)}.
\]
\end{proposition}

\begin{proof}
As in the proof of Lemma~\ref{lem:partialobservationIR}, define $\tilde\ell_t(i)$ to be a $\{0,1\}$-valued random variable with mean $\bar\ell_t(i)$, independently of everything else. Bayes rule implies: 
\begin{align*}
	\mathbb P_t\big[i\in S~|~\tilde\ell_t(i)=1\big]
	&=
	\frac{\mathbb P_t[i\in S]\cdot \mathbb P_t\big[\tilde\ell_t(i)=1~|~ i\in S\big]}{\mathbb P_t[\tilde\ell_t(i)=1]}
	\\
	&=
	 \frac{p_t(i)\bar\ell_t(i,i)}{\bar\ell_t(i)}.
\end{align*}
Rearranging, we find
\[
	\bar\ell_t(i,i)\
	=
	\frac{\mathbb P_t\big[i\in S~|~\tilde\ell_t(i)=1\big]\bar\ell_t(i)}{p_t(i)}
\]
and so
\[
	\bar\ell_t(i)-\bar\ell_t(i,i)
	=
	\bar\ell_t(i)\left(\frac{p_t(i)-\mathbb P_t\big[i\in S~|~\tilde\ell_t(i)=1\big]}{p_t(i)}\right).
\]
Therefore we may rewrite the right-hand side of the statement to be proved:
\begin{equation}
\label{eq:bayes-calc-admissible}
\begin{aligned}
	\sum_{i=1}^d \hat p_t(i)p_t(i)^2 f''(p_t(i)) & 
	\frac{\left(\bar\ell_t(i)-\bar\ell_t(i,i)\right)_+^2}{\bar\ell_t(i)}
	\\
	&= 
	\sum_{i=1}^d \hat p_t(i)\bar\ell_t(i) f''(p_t(i)) \left(p_t(i)-\mathbb P_t\big[i\in S~|~\tilde\ell_t(i)=1\big]\right)_+^2 .
\end{aligned}
\end{equation}
Below we will use the conditional probability $\mathbb P_t\big[i\in A^*~|~\tilde\ell_t(i)\big]$, which is a random variable which depends on information up to time $t$ and also on the value $\tilde\ell_t(i)\in \{0,1\}$. (This is a completely standard use of notation, but we want to clarify that it involves conditioning on the \emph{random variable} $\tilde\ell_t(i)$ instead of the \emph{event} $[\tilde\ell_t(i)=1]$ as is done just above.)
Applying \eqref{eq:admissible-taylor}, we find:
\begin{equation}
\label{eq:another-one}
\begin{aligned}
	&f\left(\mathbb P_t\big[i\in S~|~\tilde\ell_t(i)\big]
	-f(p_t(i)\right)
	\\
	&\geq 
	f'(p_t(i))\cdot\big(\mathbb P_t\big[i\in S~|~\tilde\ell_t(i)\big]-p_t(i)\big)+\frac{f''(p_t(i))}{2}\left(p_t(i)-\mathbb P_t\big[i\in S~|~\tilde\ell_t(i)\big]\right)_+^2.
\end{aligned}
\end{equation}
Note that 
\[
	\E_t\left[\mathbb P_t\big[i\in S~|~\tilde\ell_t(i)\big]\right]=p_t(i)
\]
by the tower rule for conditional expectations. Taking the expectation over $\tilde\ell_t(i)$ in \eqref{eq:another-one} yields
\begin{align*}
	\mathbb E_t\left[
		f\big(\mathbb P_t\big[i\in S~|~\tilde\ell_t(i)\big]\big)-f(p_t(i))
	\right]
	&\geq  
	\mathbb E_t\left[\frac{f''(p_t(i))}{2}\left(p_t(i)-\mathbb P_t\big[i\in S~|~\tilde\ell_t(i)\big]\right)_+^2\right]\\
	&\geq 
	\frac{f''(p_t(i))\bar\ell_t(i)}{2}
	\left(p_t(i)-\mathbb P_t\big[i\in S~|~\tilde\ell_t(i)=1\big]\right)_+^2.
\end{align*}
Multiplying by $\hat p_t(i)$ (which is determined at time $t$) and summing over $i$,
\begin{equation}
\label{eq:bandit-general-admissible}
\begin{aligned}
	\mathbb E_t&
	\left[\sum_{i=1}^d \hat p_t(i)\bar\ell_t(i) \frac{f''(p_t(i))}{2} \left(p_t(i)-\mathbb P_t\big[i\in S~|~\tilde\ell_t(i)=1\big]\right)_+^2 \right] 
	\\
	&\leq 
	\mathbb E_t\left[\sum_{i=1}^d \hat p_t(i)\left(f\big(\mathbb P_t\big[i\in S~|~\tilde\ell_t(i)\big]\big)-f(p_t(i))\right)\right].
\end{aligned}
\end{equation}
Convexity of $f$ implies that $f(X_t)$ is a submartingale for any martingale $X_t$. In particular for all $i,j\in [d]$, we have
\begin{align}
\label{eq:f-martingale-1}
	\mathbb E_t\big[f\big(\mathbb P_t\big[i\in S~|~\ell_t(j)\big]\big)\big]
	&\geq
	\mathbb E_t\big[f\big(\mathbb P_t\big[i\in S~|~\tilde\ell_t(j)\big]\big)\big]
	\\
\label{eq:f-martingale-2}
	&\geq
	f(p_t(i))
\end{align}
Combining the results above allows us to finally conclude the proof:
\begin{align*}
	\mathbb E_t[F(p_{t+1}(i))-F(p_t(i))]
	&=
	\sum_{i,j=1}^d \hat p_t(j)
	\left(
		\mathbb E_t\big[f(\mathbb P_t[i\in S~|~\ell_t(j)]-f(p_t(i))\big] 
	\right)
	\\
	&\stackrel{\eqref{eq:f-martingale-1}}{\geq}
	\sum_{i,j=1}^d \hat p_t(j)
	\left(
		\mathbb E_t\big[f(\mathbb P_t[i\in S~|~\tilde\ell_t(j)]-f(p_t(i))\big] 
	\right)
	\\
	&\stackrel{\eqref{eq:f-martingale-2}}{\geq}
	\sum_{i=1}^d \hat p_t(i)\mathbb E_t
	\left[
		f\left(
			\mathbb P_t\big[i\in S~|~\ell_t(i)\big]
		\right)
		-f(p_t(i))
	\right] 
	\\
	&\stackrel{\eqref{eq:bandit-general-admissible}}{\geq} 
	\sum_{i=1}^d 
	\hat p_t(i)\bar\ell_t(i)
	\frac{f''(p_t(i))}{2}
	\left(p_t(i)-\mathbb P_t\big[i\in S~|~\tilde\ell_t(i)=1\big]\right)_+^2
	\\
	&\stackrel{\eqref{eq:bayes-calc-admissible}}{=} 
	\sum_{i=1}^d 
	\hat p_t(i) 
	p_t(i)^2 
	f''(p_t(i))
	\left(
		\frac
		{\left(\bar\ell_t(i)-\bar\ell_t(i,i)\right)_+^2}
		{2\bar\ell_t(i)}
	\right).
\end{align*}
\end{proof}

\begin{corollary}
\label{cor:TSandMD}
Let $f$ be admissible (recall Definition~\ref{defn:admissible}), and define $\mathcal R_t,\mathcal C_t$ as in \eqref{eq:rare-common-bandit} for $\gamma>0$. Consider a bandit problem (with $m=1$) such that $L^*\leq \oL^*$ almost surely. Then Thompson sampling satisfies
\begin{equation}
\label{eq:TSandMD}
	\E[R_T]\leq \E\left[\sum_{(t,i):i\in\mathcal R_t} p_t(i)\ell_t(i)\right]+\E\left[\sum_{(t,i):i\in\mathcal C_t} p_t(i)(\ell_t(i)-\ell_t(i,i))_+\right]
\end{equation}
and the two terms are bounded by
\begin{align}
\label{eq:bandit-rare-regret}
	\E\left[\sum_{(t,i):i\in\mathcal R_t} p_t(i)\ell_t(i)\right]
	&\leq
	\min\big(\gamma T,d\cdot(2\gamma \oL^*+8\log(T)+4)\big);
	\\
\label{eq:bandit-common-regret}
	\E\left[\sum_{(t,i):i\in\mathcal C_t} p_t(i)(\ell_t(i)-\ell_t(i,i))_+\right]
	&\leq 
	\sqrt{
		2(\Max(F)-F(p_1)) \cdot 
		\E\left[
			\sum_{(t,i):i\in\mathcal C_t} \frac{\bar\ell_t(i)}{p_t(i)f''(p_t(i))}
		\right] 
	} \,.
\end{align}
\end{corollary}

\begin{proof}
The first inequality \eqref{eq:TSandMD} follows exactly as \eqref{eq:rare-common-decomposition} in the proof of Theorem~\ref{thm:banditregret}.

For the first term, the upper bound $\gamma T$ is immediate while Theorem~\ref{thm:bayesianexplorer} with $\gamma_1=\gamma_2=\gamma$ implies
\[
	\E\left[\sum_{(t,i):i\in \mathcal R_t}p_t(i)\bar\ell_t(i)\right]\leq d\cdot(2\gamma \oL^*+8\log(T)+4).
\]
For the second term,  
\begin{equation}
\label{eq:bandit-MD-bound}
\begin{aligned}
	\E&\left[\sum_{(t,i):i\in \mathcal C_t}  p_t(i)\cdot (\bar\ell_t(i)-\bar\ell_t(i,i))_+\right]
	\\ 
	&\leq 
	\sqrt{\E\sum_{(t,i):i\in\mathcal C_t} p_t(i)^3 f''(p_t(i)) \frac{\left(\bar\ell_t(i)-\bar\ell_t(i,i)\right)_+^2}{\bar\ell_t(i)} }\cdot\sqrt{\E\sum_{(t,i):i\in\mathcal C_t} \frac{\bar\ell_t(i)}{p_t(i)f''(p_t(i))} }
	\\
	&\leq 
	\sqrt{\E\sum_{(t,i)\in [T]\times [d]} p_t(i)^3 f''(p_t(i)) \frac{\left(\bar\ell_t(i)-\bar\ell_t(i,i)\right)_+^2}{\bar\ell_t(i)} }\cdot\sqrt{\E\sum_{(t,i):i\in\mathcal C_t} \frac{\bar\ell_t(i)}{p_t(i)f''(p_t(i))} }
	\\
	&\stackrel{Prop. \ref{prop:TSandMD}}{\leq}
	\sqrt{2\sum_{t=1}^T\E_t[F(p_{t+1})-F(p_{t})]}\cdot\sqrt{\E\sum_{(t,i):i\in\mathcal C_t} \frac{\bar\ell_t(i)}{p_t(i)f''(p_t(i))} }
	\\
	&\leq  
	\sqrt{2\cdot \E[F(p_T)-F(p_1)]}\cdot\sqrt{\E\sum_{(t,i):i\in\mathcal C_t} \frac{\bar\ell_t(i)}{p_t(i)f''(p_t(i))} }\\
	&\leq  \sqrt{2\cdot (\Max(F)-F(p_1))\cdot\E\sum_{(t,i):i\in\mathcal C_t} \frac{\bar\ell_t(i)}{p_t(i)f''(p_t(i))} }.
\end{aligned}
\end{equation}
Here the first inequality used Cauchy--Schwarz. The second expanded the first sum from $\{(t,i):i\in\mathcal C_t\}$ to all of $[d]\times [T]$. The third applies Proposition~\ref{prop:TSandMD} to the sum over $i$, and the fourth inequality telescopes the resulting sum. The fifth and final inequality is trivial. 

\end{proof}

Now we can prove the refined bandit estimates. We begin with the Tsallis entropy.

\begin{proof}{of Theorem~\ref{thm:MDtsallis}:}
We take $f(x)=-x^{\alpha}$. Then 
\begin{align*}
	f''(x)&=\alpha(1-\alpha)x^{\alpha-2}
	;
	\\
	\Max(F)&=-1;
	\\
	\Min(F)&=-d^{1-\alpha}.
\end{align*}
Thus Corollary~\ref{cor:TSandMD} yields
\[
	\mathbb E[R_T]
	\leq d\cdot(2\gamma \oL^*+8\log(T)+4)
	+
	\sqrt{
		2\left(-1+\sum_{i=1}^d p_1(i)^{\alpha}\right)
	}
	\cdot
	\sqrt{
		\mathbb E\sum_{(t,i):i\in\mathcal C_t} \frac{\bar\ell_t(i)p_t(i)^{1-\alpha}}{\alpha(1-\alpha)} 
	}
	.
\]
Without the square-root, the first part of the last term is
\[
	2\left(-1+\sum_{i=1}^d p_1(i)^{\alpha}\right)
	\leq 
	O(H_{\alpha}(p_1)).
\]
Removing the square-root and $\frac{1}{\alpha(1-\alpha)}$ from the second part and applying H{\"o}lder's inequality,
\begin{align*}
	\mathbb E\sum_{(t,i):i\in\mathcal C_t} \bar\ell_t(i)p_t(i)^{1-\alpha}
	&\leq 
	\left(
		\mathbb E\sum_{(t,i):i\in\mathcal C_t} \bar\ell_t(i)
	\right)^{\alpha}
	\left(
		\mathbb E\sum_{(t,i):i\in\mathcal C_t}\bar\ell_t(i)p_t(i)
	\right)^{1-\alpha}
	\\
	&\leq  
	d^{\alpha}\left( 
		\oL^*+2\left(\log\left(\frac{1}{\gamma}\right)+10\right)\sqrt{\frac{\oL^*}{\gamma}} 
	\right)^{\alpha}
	\cdot \mathbb E[L_T]^{1-\alpha}.
\end{align*}
Here we used the fact that for each $i\in [d]$,
\begin{equation}
\label{eq:loss-formula}
	\E\left[\sum_{t=1}^T \bar\ell_t(i)p_t(i)\right]=\E[L_T]
\end{equation}
is the expected loss incurred by Thompson sampling. With the choice $\gamma=\frac{\log^2(\oL^*)}{\oL^*}$, we have
\[
	\oL^*+2\left(\log\left(\frac{1}{\gamma}\right)+10\right)\sqrt{\frac{\oL^*}{\gamma}}\leq O(\oL^*).
\]
Assuming $\E[R_T]\geq 0$ (else any regret statement is vacuous), we get 
\begin{align*}
	\mathbb E[R_T]
	&\leq 
	d\cdot
	\big(
		2\log^2(\oL^*)+8\log(T)+4
	\big)
	+
	O\big(\sqrt{H_{\alpha}(p_1)}\big)
	\left(d\oL^*\right)^{\alpha/2}
	\mathbb E[L_T]^{(1-\alpha)/2}
	\\
	&\leq 
	d\cdot\big(2\log^2(\oL^*)+8\log(T)+4)
	+
	O\big(\sqrt{H_{\alpha}(p_1) d^{\alpha}}\big)
	\cdot
	\left(\oL^*+\mathbb E[R_T]\right)^{1/2}.
\end{align*}
We finally apply Lemma~\ref{lem:selfbounding} below with:
\begin{itemize}
	\item $R=\mathbb E[R_T]$
	\item $X=d\cdot(2\gamma \oL^*+8\log(T)+4)$
	\item $Y=O(\sqrt{H_{\alpha}(p_1)d^{\alpha}})$
	\item $Z=\oL^*$.
\end{itemize}
This gives the regret bound
\[
	\mathbb E[R_T]=\frac{1}{\sqrt{\alpha(1-\alpha)}}
	\cdot 
	O\left(\sqrt{H_{\alpha}(p_1) d^{\alpha}\oL^*}+H_{\alpha}(p_1) d^{\alpha}+d\log^2(\oL^*)+d\log(T)\right).
\]
Observing that $H_{\alpha}(p_1)d^{\alpha}\leq d\leq d\log T$ allows us to remove the $H_{\alpha}(p_1)d^{\alpha}$ term and thus completes the proof.
\end{proof}

\begin{lemma}
\label{lem:selfbounding}

If $R,X,Y,$ and $Z$ are non-negative real numbers and $R\leq X+Y\sqrt{Z+R}$, then
\[
	R\leq  X+Y^2+Y\sqrt{Z}.
\]
\end{lemma}

\begin{proof}
Rearranging, squaring, and further rearranging yields:
\begin{align*}
	R
	&\leq 
	X+Y\sqrt{Z+R} 
	\\
	\implies  
	R^2-2RX+X^2 
	&\leq  
	Y^2 Z+Y^2 R 
	\\
	\implies 
	R^2-(2X+Y^2)R 
	&\leq  
	Y^2Z-X^2 
	\\
	\implies 
	\left(R-\left(X+\frac{Y^2}{2}\right)\right)^2&\leq  \frac{Y^4}{4}+Y^2Z-X^2
	\\
	\implies R
	&\leq  
	X+\frac{Y^2}{2}+\sqrt{\frac{Y^4}{4}+Y^2Z-X^2} 
	\\
	&\leq  X+Y^2+Y\sqrt{Z}.
\end{align*}
\end{proof}

\begin{proof}{of Theorem~\ref{thm:MDlogbarrier}:}

We apply Corollary~\ref{cor:TSandMD} again, this time with $f(x)=-\log(Tx+1)$. We have
\begin{align*}
	\diam(F)&= d\log(T)(1+o(1));
	\\
	f''(x)&=\frac{1}{(x+T^{-1})^2}.
\end{align*}
Define $\mathcal R_t$ and $\mathcal C_t$ using \eqref{eq:rare-common-bandit} with $\gamma=T^{-1}$. The $\mathcal R_t$ contribution in Corollary~\ref{cor:TSandMD} is at most $\gamma T=1$ so it remains to estimate the $\mathcal C_t$ contribution.

To do this we observe that for $p_t(i)\geq \gamma=\frac{1}{T}$, 
\[
	f''(p_t(i))^{-1}= (p_t(i)+T^{-1})^2\leq p_t(i)^2 + 3p_t(i)T^{-1}.
\]
Plugging in this estimate gives 
\[
	\sum_{t,i:i\in\mathcal C_t} \frac{\bar\ell_t(i)}{p_t(i)f''(p_t(i))}\leq \sum_{t,i}\left(p_t(i)\bar\ell_t(i) + \frac{3\bar\ell_t(i)}{T}\right)\leq\E[L_T]+ 6d.
\]
Going back to the beginning and combining, we have shown
\begin{equation}
\label{eq:log-barrier-almost}
	\E[R_T]=O\left( \sqrt{d\log(T) (\E[L_T]+d)}+1\right)=O\left(\sqrt{d\mathbb E[L_T]\log(T)}+d\sqrt{\log(T)}\right).
\end{equation}
Recall from Lemma~\ref{lem:self-bound} $a-b\leq \sqrt{ac}$ implies $a-b\leq \sqrt{bc}+c$ for non-negative $a,b,c$. The proof is concluded by taking
\begin{itemize}
	\item $a=\mathbb E[L_T]$
	\item $b=\mathbb E[L^*]+O(d\sqrt{\log(T)})$
	\item $c=O(d\log(T))$
\end{itemize}
to obtain
\[
	\E[R_T]=O\left(\sqrt{d\E[L^*]\log(T)}+d\log(T)\right).
\]
\end{proof}

\section{Combinatorial Semi-bandit Setting}\label{sec:combbandit}

We now consider semi-bandit feedback in the combinatorial setting, combining the intricacies of Sections~\ref{sec:combinatorial} and~\ref{sec:bandit}. We again have an action set $\mathcal A$ contained in the set $\{a\in \{0,1\}^d:||a||_1=m\}$, but now we observe the $m$ losses of the arms played. A natural generalization of the bandit $m=1$ proof to higher $m$ yields a first-order regret bound of $\tilde O(\sqrt{md\oL^*})$. However, we give a refined analysis using an additional trick of ranking the $m$ arms in $a^*$ by their total loss and performing an information theoretic analysis on a certain set partition of these $m$ optimal arms. This method allows us to obtain a $\tilde O(\sqrt{d\oL^*})$ regret bound for the semi-bandit regret. The analyses based on other mirror maps extend as well.

\subsection{Naive Analysis and Intuition}

We let $A^*\in\mathcal A$ be the optimal set of $m$ arms, and assume that $A$ has total loss $L^*\leq \oL^*$. Extending the definition before, let
\begin{equation}
\label{eq:elltij-sembandit}
	\bar\ell_t(i,j)=\E[\ell_t(i)|j\in A^*].
\end{equation}
Ignoring the issue of exactly how to assign arms as rare/common, one expects that mimicking the proof of Theorem~\ref{thm:banditregret} will imply:
\begin{align*}
	\E[R_T]&=\E\left[\sum_{t,i} p_t(i)(\ell_t(i)-\ell_t(i,i))\right]\\
	&\leq \E\left[\sum_{(t,i):i\in \mathcal R_t} p_t(i)\ell_t(i)\right] + \E\left[\sum_{(t,i):i\in \mathcal C_t} p_t(i)\cdot (\bar\ell_t(i)-\bar\ell_t(i,i))_+ \right].
\end{align*}
(Proposition~\ref{prop:bandit-regret-basic} extends easily to the semibandit setting; see below for a careful statement.)
The first term is again small due to Theorem~\ref{thm:bayesianexplorer}A and the second term can be estimated by mimicking \eqref{eq:common-CS-bandit} and applying Cauchy--Schwarz to obtain
\[
	\E\left[\sum_{(t,i):i\in \mathcal C_t} p_t(i)\cdot (\bar\ell_t(i)-\bar\ell_t(i,i))_+ \right] \leq 2 \E\left[\sum_{(t,i):i\in\mathcal C_t}\ell_t(i)\right]\cdot H^c(A^*).
\]
The main difference is that now the coordinate entropy $H^c(A^*)$ can be as large as $\tilde O(m)$. So the result is
\[
	\E[R_T]\leq \tilde O\left(\sqrt{H^c(a^*)d\oL^*}\right)
	= 
	\tilde O\left(\sqrt{md\oL^*}\right).
\]
This argument is inefficient because it allows every arm to have loss $\oL^*$ before becoming rare. However actually, only $j$ optimal arms can have loss more than $\frac{\oL^*}{j}$. So although the coordinate entropy of $A^*$ can be as large as $\tilde O(m)$, the coordinate entropy on the arms with large loss so far is much smaller. This motivatives the rank ordering introduced in the next subsection.

Before moving on, let us justify the first step of the attempt above by generalizing Proposition~\ref{prop:bandit-regret-basic}. We give a careful statement but omit the proof as it is exactly identical. Recall the notation \eqref{eq:elltij-sembandit}. 

\begin{proposition}
\label{prop:semibandit-regret-basic}
Suppose an algorithm for the semibandit game has $p_t(i)=\mathbb P_t[i\in A^*]$ and $\hat p_t=\mathbb P_t[i\in A_t]$. Then the expected regret is given by
\[
	R_T=\sum_{t=1}^T r_t
\]
for
\[
	r_t=\sum_{i=1}^d \big(\hat p_t(i)\bar\ell_t(i)-p_t(i)\bar\ell_t(i,i))\big).
\]
In the case $\hat p_t=p_t$ of Thompson sampling, this formula simplifies to
\[
	r_t=\sum_{i=1}^d \big(p_t(i)(\bar\ell_t(i)-\bar\ell_t(i,i))\big).
\]
\end{proposition}

\subsection{Rare Arms and Rank Order}

We introduce two notions needed for the semi-bandit proof. First, analogously to our definition of rare and common arms in the bandit $m=1$ case, we partition $[d]$ into rare and common arms. The definition becomes slightly more complicated in the combinatorial setting, since setting some arms to be rare can affect probabilities for other arms.

We construct $\mathcal R_t$ and $\mathcal C_t$ starting with an empty subset $\mathcal R_t=\emptyset\subseteq [d]$ of rare arms and grow it as follows. While there exists $i\in [d]$ satisfying 
\begin{equation}
\label{eq:semibandit-partition-alg}
	\mathbb P_t\big[(i\in A^*)\text{ and } A^*\subseteq \mathcal C_t\big]\leq \gamma,
\end{equation}
we choose such an arm $i$ to add to $\mathcal R_t$. (Here $\mathcal C_t=[d]\backslash\mathcal R_t$ at all stages during the algorithm. At the end, all $i\in \mathcal C_t$ do \textbf{not} satisfy \eqref{eq:semibandit-partition-alg}. In particular,
\begin{equation}
\label{eq:semibandit-partition-Ct}
	\mathbb P_t\big[(i\in A^*)\text{ and } A^*\subseteq \mathcal C_t\big]> \gamma
	\quad\quad
	\forall i\in\mathcal C_t.
\end{equation}
Otherwise stated, we obtain a subset $C_t\subseteq [d]$ of arms, each of which has a large probability at least $\gamma$ to be in $A^*$, even after removing actions which overlap $\mathcal R_t$ at all. In addition to \eqref{eq:semibandit-partition-Ct}, the resulting partition $[d]=\mathcal R_t\cup \mathcal C_t$ satisfies the following.
For all $i\in \mathcal R_t$, 
\begin{equation}
\label{eq:semibandit-partition-Rt}
	p_t(i)\leq 
	\mathbb P[A^*\not\subseteq \mathcal C_t]\leq 
	d\gamma.
\end{equation}
This is because each time an arm $i\in [d]$ moves from $\mathcal C_t$ to $\mathcal R_t$ in the algorithm above, the quantity $\mathbb P[A^*\not\subseteq \mathcal C_t]$ increases by at most $\gamma$. Comparing with the conditions after \eqref{eq:rare-common-partition} suggests that in semi-bandit situations we should take $(\gamma_1,\gamma_2)=(\gamma,d\gamma)$ in applying Theorem~\ref{thm:bayesianexplorer}. This is exactly what we will do.

The next step is to implement a rank ordering of the $m$ coordinates. We take
\[
	A^*=\{a_1^*,a_2^*,\dots,a_m^*\}
\] 
where
\[
	L_T(a_1^*)\geq L_T(a_2^*)\geq \dots \geq L_T(a_m^*)
\]
and ties are broken arbitrarily. Crucially, we observe that 
\begin{equation}
\label{eq:aj-loss}
	L_T(a_j^*)\leq \frac{\oL^*}{j}.
\end{equation}
We further consider a general partition of $[m]$ into disjoint subsets $S_1,S_2,\dots,S_r$. Define 
\[
	A_{S_k}^*=\{a_s^*:s\in S_k\}.
\]
We will carry out an information theoretic argument which treats separately the events $\{i\in A_{S_k}^*\}$. At the end of the calculation, we will see that the dyadic partition $S_k=\{2^{k-1},\dots,2^k-1\}$ improves the naive analysis above. In fact the naive analysis corresponds to the trivial partition $S_1=[m]$. Towards such an analysis it will be helpful to define
\begin{align}
\label{eq:pti-Sk}
	p_t(i,S_k)
	&=
	\mathbb P[i\in A^*_{S_k}];
	\\
\label{eq:ell-Sk}
	\bar\ell(i,S_k)
	&=
	\E[\ell_t(i)~|~i\in A_{S_k}].
\end{align}

\subsection{Semi-bandit Regret Bound via Shannon Entropy}

Here we carry out the strategy just outlined for the Shannon entropy. We again begin by decomposing the regret into contributions from $\mathcal R_t$ and $\mathcal C_t$. We choose a small threshold $\gamma\in [0,1/d]$ and apply the recursive procedure from the previous section, thus obtaining partitions $[d]=\mathcal R_t\cup\mathcal C_t$ which satisfy \eqref{eq:semibandit-partition-Ct} and \eqref{eq:semibandit-partition-Rt}. We then apply Theorem~\ref{thm:bayesianexplorer} with $(\gamma_1,\gamma_2)=(\gamma,d\gamma)$ to bound the resulting terms. 

\begin{restatable}{theorem}{semibanditshannon}
\label{thm:semibanditshannon}
The expected regret of Thompson Sampling in the semi-bandit setting is 
\[
	O\left(\log(m)\sqrt{d\oL^*\log(d)}+md^2\log^2(\oL^*)+d\log(T)\right).
\]
\end{restatable}

\begin{proof}
Set 
\[
	(\gamma_1,\gamma_2)=\left(\frac{m\log^2(\oL^*)}{\oL^*},\frac{md\log^2(\oL^*)}{\oL^*}\right).
\] 
Let $S_1,\dots,S_r$ be as discussed in the previous subsection. The analysis begins with another decomposition of the regret into rare and common contributions. Recall Proposition~\ref{prop:semibandit-regret-basic} and the notations \eqref{eq:pti-Sk} and \eqref{eq:ell-Sk}. We have:
\begin{equation}
\label{eq:semibandit-start}
\begin{aligned}
	\E[R_T] 
	&\leq 
	\E\left[\sum_{(t,i):i\in \mathcal R_t} p_t(i)\bar\ell_t(i) \right] 
	+ 
	\mathbb E\left[\sum_{(t,i):i\in \mathcal C_t} p_t(i)(\bar\ell_t(i)-\bar\ell_t(i,i))\right]
	\\
	&= 
	\mathbb E\left[\sum_{(t,i):i\in \mathcal R_t} p_t(i)\bar\ell_t(i) \right] 
	+ 
	\sum_{k=1}^r\mathbb E\left[\sum_{(t,i):i\in \mathcal C_t} p_t(i,S_k)(\bar\ell_t(i)-\bar\ell_t(i,S_k))\right].
\end{aligned}
\end{equation}
A direct application of Theorem~\ref{thm:bayesianexplorer}A gives the bound
\begin{equation}
\label{eq:semibandit-rare-bound}
	\mathbb E\left[\sum_{(t,i):i\in \mathcal R_t} p_t(i)\bar\ell_t(i) \right] 
	\leq
	O\left(md^2\log^2(\oL^*)+d\log(T)\right)
\end{equation}
for the first term on the right-hand side. 
For the second term, we apply Cauchy--Schwarz for each $k\in [r]$ separately. This yields
\begin{equation}
\label{eq:idk}
\begin{aligned}
	&\sum_{k=1}^r\mathbb E\left[\sum_{(t,i):i\in \mathcal C_t} p_t(i,S_k)(\ell_t(i)-\ell_t(i,S_k))\right]
	\\
	&\leq
	\sum_{k=1}^r 
	\sqrt{
		\E\left[\sum_{(t,i)} p_t(i)p_t(i,S_k)\left(\frac{(\bar\ell_t(i)-\bar\ell_t(i,S_k))_+^2}{\bar\ell_t(i)}\right)\right]^{1/2}  
		\E\left[\sum_{(t,i):i\in \mathcal C_t} \frac{p_t(i,S_k)\bar\ell_t(i)}{p_t(i)}\right]^{1/2}
	}.
\end{aligned}
\end{equation}
By Lemma~\ref{lem:scalesensitivebanditIR} the first expectation inside the square-root can be estimated information theoretically by $H^c(A^*_{S_k})$:
\begin{align*}
	\E\left[\sum_{(t,i)} p_t(i)p_t(i,S_k)\left(\frac{(\bar\ell_t(i)-\bar\ell_t(i,S_k))_+^2}{\bar\ell_t(i)}\right)\right]
	&\leq 
	2\sum_t I_t^c[S_k]
	\\
	&\leq 
	2\cdot H^c(A_{S_k}^*).
\end{align*}
Moreover we can change $\ell_t(i)$ to $\bar\ell_t(i)$:
\[
	\E\left[\sum_{(t,i):i\in \mathcal C_t} \frac{p_t(i,S_k)\bar\ell_t(i)}{p_t(i)}\right]
	=
	\E\left[\sum_{(t,i):i\in \mathcal C_t} \frac{p_t(i,S_k)\ell_t(i)}{p_t(i)}\right]
\]
This is because $p_t$ are probabilities at the start of round $t$ and $\bar\ell_t(i)=\E_t[\ell_t(i)]$. Substituting into \eqref{eq:idk}, the common-arm regret term is upper-bounded by:
\[
	\sum_{k=1}^r\mathbb E\left[\sum_{(t,i):i\in \mathcal C_t} p_t(i,S_k)(\ell_t(i)-\ell_t(i,S_k))\right]
	\leq
	\sum_{k=1}^r\sqrt{2\cdot H^c(A^*_{S_k}) \E\left[\sum_{(t,i):i\in \mathcal C_t} \frac{p_t(i,S_k)\ell_t(i)}{p_t(i)}\right]}.
\]
The reason for introducing the sets $S_k$ now appears: to give a separate estimate for the inner expectation on the right-hand side. Let $s_k=\min(S_k)$. Observe that if $L_t(i)>\frac{\oL^*}{s_k}$, then we cannot have $i\in A^*_{S_k}$ because 
\[
	L_t(a_j^*)\leq L_T(a_j^*)\leq \frac{\oL^*}{j}<L_t(i),
	\quad\quad
	\forall j\in S_k.
\] 
Roughly speaking, for each fixed $i$ the sum
\[
	\sum_{t\in [T]:\text{ }i\in\mathcal C_t} \frac{p_t(i,S_k)\bar\ell_t(i)}{p_t(i)}
\]
will typically stop growing much once $L_t(i)>\frac{\oL^*}{s_k}$ because $p_t(i,S_k)$ will be very small while $p_t(i)\geq \gamma$. Before this starts to happen, we have the simple estimate $\frac{p_t(i,S_k)}{p_t(i)}\leq 1$. Therefore the sum should be bounded by approximately $\frac{\oL^*}{s_k}$. In fact Lemma~\ref{lem:semibanditTS} below gives the estimate
\[
	\E\left[\sum_{t\in [T]:\text{ }i\in\mathcal C_t} \frac{p_t(i,S_k)\bar\ell_t(i)}{p_t(i)}\right]
	\leq
	\frac{\oL^*}{s_k}+O\left(\log(1/\gamma_1)\sqrt{\frac{\oL^*}{s_k\gamma_1}}\right).
\]
Using the estimate $H^c(A^*_{S_k})=O(|S_k|\log(d))$ and multiplying by $d$ to account for the $d$ arms, the common arm regret contribution is hence estimated by
\begin{align}
\nonumber
	\sum_{k=1}^r\mathbb E &\left[\sum_{(t,i):i\in \mathcal C_t} p_t(i,S_k)(\ell_t(i)-\ell_t(i,S_k))\right]
	\leq
	\sum_{k=1}^r\sqrt{2\cdot H^c(A^*_{S_k}) \E\left[\sum_{(t,i):i\in \mathcal C_t} \frac{p_t(i,S_k)\ell_t(i)}{p_t(i)}\right]}
	\\
\label{eq:?}
	&\leq
	O\left(\sum_{k=1}^r
	\sqrt{
		2d\log(d)|S_k|\oL^*
	\cdot 
	\left(s_k^{-1}+\log(1/\gamma_1)\big(s_k\gamma_1\oL^*\big)^{-1/2}\right)
	}
	\right).
\end{align}
Because $\gamma_1=\frac{m\log^2{\oL^*}}{\oL^*}$ it follows that
\[
	\log(1/\gamma_1)\big(s_k\gamma_1\oL^*\big)^{-1/2}
	=
	O\left(\sqrt{\frac{1}{s_k m}}\right).
\]
Next we substitute and observe that 
\[
	\sqrt{s_k^{-1}+O\left(\sqrt{\frac{1}{s_k m}}\right)}=O(s_k^{-1/2})
\] 
since $s_k\leq m$.
Therefore the right-hand side \eqref{eq:?} above is bounded by
\begin{equation}
\label{eq:another-bound}
\begin{aligned}
	&O\left(\sum_{k=1}^r
	\sqrt{
		2d\log(d)|S_k|\oL^*
	\cdot 
	\left(s_k^{-1}+\log(1/\gamma_1)\big(s_k\gamma_1\oL^*\big)^{-1/2}\right)
	}
	\right)
	\\
	&\leq
	O\left(\sum_{k=1}^r\sqrt{2d\log(d)|S_k|\oL^* s_k^{-1}}\right)
	\\
	&=
	\sqrt{d\log(d)\oL^*}\cdot O\left(\sum_{k=1}^r\sqrt{\frac{|S_k|}{s_k}}\right).
\end{aligned}
\end{equation}

We are left with finding a partition $(S_1,\dots,S_r)$ that makes the right-hand sum $\sum_{k=1}^r\sqrt{\frac{|S_k|}{s_k}}$ as small as possible. Taking a single set $S_1=[m]$ as in the naive analysis gives $\sqrt{m}$, and taking $d$ singleton subsets $S_k=\{k\}$ also yields $\sum_{k=1}^m k^{-1/2}=\Theta(\sqrt{m})$. But a dyadic decomposition does much better! Setting
\begin{equation}
\label{eq:dyadic}
	S_k=\{2^{k-1},\dots, 2^k-1\}\cap [m]
\end{equation}
for $k\leq \lceil\log_2(m)\rceil$, we find
\[
	\sum_{k\leq \lceil\log_2(m)\rceil}\sqrt{\frac{|S_k|}{s_k}}
	\leq 
	\sum_{k\leq \lceil\log_2(m)\rceil} \sqrt{2} = O(\log m).
\] 
Combined with \eqref{eq:?} and \eqref{eq:another-bound}, this choice thus gives
\begin{align*}
	\sum_{k=1}^r\mathbb E &
	\left[\sum_{(t,i):i\in \mathcal C_t} p_t(i,S_k)(\ell_t(i)-\ell_t(i,S_k))\right]
	\\
	&\leq
	\sum_{k=1}^r
	\sqrt{
		2d\log(d)|S_k|\oL^*
	\cdot 
	\left(s_k^{-1}+\log(1/\gamma_1)\big(s_k\gamma_1\oL^*\big)^{-1/2}\right)
	}
	\\
	&\leq
	\sqrt{d\log(d)\oL^*}\cdot O\left(\sum_{k=1}^r\sqrt{\frac{|S_k|}{s_k}}\right)
	\\
	&\leq
	O\left(\log(m)\sqrt{d\log(d)\oL^*}\right).
\end{align*}
Combining with the estimate \eqref{eq:semibandit-rare-bound} for rare arms and substituting into \eqref{eq:semibandit-start} finishes the proof.
\end{proof}








\subsection{Semi-bandit Regret Bound from Tsallis Entropy}

We improve the regret bound of Theorem~\ref{thm:semibanditshannon} using Tsallis entropy. The main result follows.

\begin{restatable}{theorem}{semibanditTS}
\label{thm:semibanditTS}
Suppose that the best combinatorial action almost surely has total loss at most $\oL^*$. Then Thompson sampling with semi-bandit feedback obeys the regret estimate
\[
	\E[R_T]\leq O\left(\log(m)\sqrt{d\oL^*}+md^2\log^2(\oL^*)+d\log(T)\right).
\]

\end{restatable}

In proving Theorem~\ref{thm:semibanditTS} we require the technical Lemma~\ref{lem:bayesianexplorerC} which is proved in the Appendix. It relies on Freedman's martingale concentration inequality.

\begin{restatable}{lemma}{bayesianexplorerC}
\label{lem:bayesianexplorerC}

Fix an arm $i\in [d]$. In the context of Theorem~\ref{thm:bayesianexplorer}, fix constants $\lambda\geq 2$ and $\tilde L>0$ and assume $\gamma_1\geq 1/\tilde L$. With probability at least $1-2e^{-\lambda/2}$, for all $t$ such that $L_t^{\mathcal C}(i)\leq \tilde L$:
\[
	U_t^{\mathcal C}(i)\leq L_t^{\mathcal C}(i)+\lambda\sqrt{\frac{\tilde L}{\gamma_1}}.
\]
\end{restatable}

The following simple result will also be useful.

\begin{lemma}
\label{lem:doob-integral}
Let $(M_t)_{t\in\mathbb Z_+}$ be a martingale started at $M_1=p\in [0,1]$ such that almost surely, $M_t\in [0,1]$ for all $t$. Then the expected maximum is
\[
	\E[\sup_{t\geq 0}M_t]\leq p(1-\log p).
\]
\end{lemma}
\begin{proof}
By Doob's inequality, 
\[
	\mathbb P[\sup_{t\geq 0}M_t\geq q]\leq p/q,
	\quad\quad
	\forall q\in [p,1].
\]
The tail-sum formula thus implies
\begin{align*}
	\E[\sup_{t\geq 0}(M_t)]
	&=
	\int_0^1 
	\mathbb P[\sup_{t\geq 0}(M_t)\geq q]\rmd q
	\\
	&\leq
	p+\int_p^1 p/q \rmd q
	\\
	&=
	p(1-\log p)
\end{align*}
as desired.
\end{proof}

\begin{lemma}
\label{lem:semibanditTS}
Fix a subset $S_k\subseteq [m]$, let $s_k=\min(S_k)$, and assume 
\[
	m/\oL^*\leq \gamma_1\leq \frac{1}{2}.
\]
Then any Bayesian bandit algorithm satisfies
\[
	\E\left[\sum_{t\in [T]:\text{ }i\in\mathcal C_t} \frac{p_t(i,S_k)\bar\ell_t(i)}{p_t(i)}\right] 
	\leq 
	\frac{\oL^*}{s_k}+O\left(\log\left(\frac{1}{\gamma_1}\right)\sqrt{\frac{\oL^*}{s_k\gamma_1}}\right).
\]
\end{lemma}

\begin{proof} 
Recall the notation of Table~\ref{table:loss-notation}.
We first apply Lemma~\ref{lem:bayesianexplorerC} with $\gamma_1=\gamma$ and
\[
	\tilde L=\frac{\oL^*}{s_k}.
\]
The conclusion is that for $\lambda\geq 2$ and $\gamma_1\geq s_k/\oL^*$, with probability at least $1-2e^{-\lambda/2}$, all $t$ with $L_t^{\mathcal C}(i)\leq \frac{\oL^*}{s_k}$ also satisfy
\begin{align*}
	U_t^{\mathcal C}(i)
	&\leq L_t^{\mathcal C}(i)+\lambda\sqrt{\frac{\oL^*}{s_k \gamma_1}}
	\\
	&\leq \frac{\oL^*}{s_k}+\lambda\sqrt{\frac{\oL^*}{s_k \gamma_1}}.
\end{align*}
Note that $p_t(i,S_k)\leq p_t(i)$ and for $i\in \mathcal C_t$ also $\gamma_1\leq p_t(i)$. It follows that for any $C>0$:
\begin{equation}
\label{eq:semibandit-common-bound-again}
	\E\left[\sum_{t\in [T]:\text{ }i\in\mathcal C_t} \frac{p_t(i,S_k)\bar\ell_t(i)}{p_t(i)}\right] \leq C +1 + \left(\frac{1}{\gamma_1}\right)\E\left[\sum_{t\in [T]:\text{ }i\in\mathcal C_t, L_t^{\mathcal C}(i)\geq C}p_t(i,S_k)\bar\ell_t(i)\right].
\end{equation}
We rewrite the latter expectation, then essentially rewrite it again as a Riemann-Stieltjes integral. Letting $p_t(i,S_k)=p_{\lfloor t\rfloor}(i,S_k)$ for any positive real $t$,
\begin{align*}
	\E\left[\sum_{t\in [T]:\text{ }i\in\mathcal C_t, L_t^{\mathcal C}(i)\geq C}p_t(i,S_k)\bar\ell_t(i)\right]
	&= 
	\E\left[\sum_{L_t^{\mathcal C}(i)\geq C}p_t(i,S_k)\ell_t^{\mathcal C}(i)\right]\\
	&\leq 
	\E\left[\int_C^{\infty} p_t(i,S_k) \rmd L_t^{\mathcal C}(i)\right].
\end{align*}
Define $\tau_x$ to be the first value of $t$ satisfying 
\[
	L_t^{\mathcal C}(i)\geq x,
\]
where $\tau_x=\infty$ if $L_T^{\mathcal C}(i)<x$. Since $\ell_t(i)\leq 1$ almost surely for all $t$, it follows that $t\geq \tau_{L_t^{\mathcal C}(i)-1}$. Therefore, changing variables from $t$ to $L_t^{\mathcal C}(i)$ yields:
\begin{equation}
\label{eq:int-dx}
\begin{aligned}
	\E\left[\int_C^{\infty} p_t(i,S_k) dL_t^{\mathcal C}(i)\right]
	&\leq 
	\E\left[\int_C^{\infty} \max_{t\geq \tau_{x-1}}(p_t(i,S_k))\cdot 1_{\tau_x<\infty}\rmd x\right]
	\\
	&\leq 
	\E\left[ \int_C^{\infty} \max_{t\geq \tau_{x-1}}(p_t(i,S_k))\cdot 1_{\tau_{x-1}<\infty}\rmd x \right]
	\\
	&\leq 
	1+\E\left[ \int_C^{\infty} \max_{t\geq \tau_{x}}(p_t(i,S_k))\cdot 1_{\tau_{x}<\infty}\rmd x \right].
\end{aligned}
\end{equation}
To translate the result of Lemma~\ref{lem:bayesianexplorerC}, we choose $x$ and $\lambda>2$ to satisfy 
\begin{equation}
\label{eq:lambda-x}
	x=\frac{\oL^*}{s_k}+\lambda\sqrt{\frac{\oL^*}{s_k \gamma_1}}
\end{equation}
Then Lemma~\ref{lem:bayesianexplorerC} implies
\begin{equation}
\label{eq:tail-bound-last}
	\E[p_{\tau_x}(i,S_k)1_{\tau_x<\infty}]\leq 2e^{-\lambda/2}.
\end{equation}
Moreover Lemma~\ref{lem:doob-integral} implies
\begin{equation}
\label{eq:doob-Sk}
	\E_{\tau_x}[\max_{t\geq \tau_x}p_t(i,S_k)1_{\tau_x<\infty}]
	\leq 
	p_{\tau_x}(i,S_k)\cdot\left(1-\log\left(p_{\tau_x}(i,S_k)\right)\right)\cdot 1_{\tau_x<\infty}.
\end{equation}
The function $f(x)=x(1-\log x)$ is increasing and concave with $f(0)=0$. We set $y=p_{\tau_x}(i,S_k)$. Using optional stopping, \eqref{eq:doob-Sk}, Jensen's inequality, and finally \eqref{eq:tail-bound-last}, we obtain
\begin{equation}
\label{eq:bound-lambda}
\begin{aligned}
	\E[\max_{t\geq \tau_x}p_t(i,S_k)1_{\tau_x<\infty}]
	&=
	\E\big[\E_{\tau_x}[\max_{t\geq \tau_x}p_t(i,S_k)1_{\tau_x<\infty}]\big]
	\\
	&\stackrel{\eqref{eq:doob-Sk}}{\leq}
	\E\big[f(y)\big]
	\\
	&\leq
	f(\E[y])
	\\
	&\stackrel{\eqref{eq:tail-bound-last}}{\leq}
	f(2e^{-\lambda/2})
	\\
	&\leq \lambda e^{-\lambda/2}.
\end{aligned}
\end{equation}

Setting 
\[
	C=\frac{\oL^*}{s_k}+10\log\left(\frac{1}{\gamma_1}\right)\sqrt{\frac{\oL^*}{s_k \gamma_1}},
\]
we use \eqref{eq:bound-lambda}, changing variables in \eqref{eq:int-dx} from integrating over $\lambda$ to integrating over $x$. This yields the estimate
\[
	\E\left[ \int_C^{\infty} \max_{t\geq \tau_{x}}(p_t(i,S_k))\cdot 1_{\tau_{x}<\infty}\rmd x \right]
	\leq
	\sqrt{\frac{\oL^*}{s_k\gamma_1}}\int_{10\log(1/\gamma_1)}^{\infty} \lambda e^{-\lambda/2} \rmd\lambda.
\]
The integral is bounded by $O(1)$ since $\gamma_1\leq\frac{1}{2}$ and also $10\log(1/\gamma_1)\geq 2$. (The latter bound is required because the above estimates only holds for $\lambda>2$, which is due to the condition in Lemma~\ref{lem:bayesianexplorerC}.) Recalling our calculations starting from \eqref{eq:semibandit-common-bound-again}, we find
\[
	\E\left[\sum_{t\in [T]:\text{ }i\in\mathcal C_t} \frac{p_t(i,S_k)\bar\ell_t(i)}{p_t(i)}\right]\leq \frac{\oL^*}{s_k}+O\left(\log\left(\frac{1}{\gamma_1}\right)\sqrt{\frac{\oL^*}{s_k\gamma_1}}\right).
\]
This completes the proof.
\end{proof}

The next lemma is used also in the log-barrier based regret bound. Recall from \eqref{eq:diam-j} that $\diam_j(F)$ is the diameter of $F=\sum_{i=1}^d f(x_i)$ restricted to $\{x\in [0,1]^d,\sum_{i=1}^d x_i=j\}$.

\begin{lemma}
\label{lem:shouldhavelabel}
Let $f$ be admissible (recall Definition~\ref{defn:admissible}), and $\mathcal R_t,\mathcal C_t$ be generated by $(\gamma_1,\gamma_2)$ (recall \eqref{eq:semibandit-partition-alg} and below). Let $S_1\cup\dots\cup S_r=[d]$ be a rank-order partition. Thompson Sampling for the semibandit problem satisfies
\begin{equation}
\label{eq:rank-order-total-regret}
	\E[R_T] \leq \mathbb E\left[\sum_{(t,i):i\in \mathcal R_t} p_t(i)\bar\ell_t(i) \right] + \mathbb E\left[\sum_{(t,i,k):i\in \mathcal C_t} p_t(i,S_k)(\bar\ell_t(i)-\bar\ell_t(i,S_k))\right]
\end{equation}
where 
\begin{align}
\label{eq:rank-order-rare-regret}
	\mathbb E\left[\sum_{(t,i):i\in \mathcal R_t} p_t(i)\bar\ell_t(i) \right]
	&\leq 
	\min\left(\gamma_2T,md^2\log^2(\oL^*)+d\log(T)\right);
	\\
\label{eq:rank-order-common-regret}
	\mathbb E\left[\sum_{(t,i,k):i\in \mathcal C_t} p_t(i,S_k)(\bar\ell_t(i)-\bar\ell_t(i,S_k))\right]
	&\leq 
	\sum_{k=1}^r\sqrt{2\cdot \diam_{|S_k|}(F)\cdot \E\sum_{(t,i):i\in\mathcal C_t} \frac{\bar\ell_t(i)}{p_t(i)f''(p_t(i,S_k))} }.
\end{align}
\end{lemma}

\begin{proof}
The inequality \eqref{eq:rank-order-total-regret} is clear while \eqref{eq:rank-order-rare-regret} follows from Theorem~\ref{thm:bayesianexplorer}, so we focus on \eqref{eq:rank-order-common-regret}. Fix $k\in [r]$ and as before for all $i\in [d]$ let
\[
	p_t(i,S_k)=\mathbb P^t[i\in S_k].
\]
Then the calculation (whose justification is identical to the $m=1$ setting in \eqref{eq:bandit-MD-bound}) goes:
\begin{align}
\nonumber
	\E
	&
	\left[\sum_{(t,i):i\in \mathcal C_t}  p_t(i,S_k)\cdot (\bar\ell_t(i)-\bar\ell_t(i,S_k))\right]
	\\ 
\nonumber
	&\leq
	\E
	\left[\sum_{(t,i):i\in \mathcal C_t}  p_t(i,S_k)\cdot (\bar\ell_t(i)-\bar\ell_t(i,S_k))_+\right]
	\\
\nonumber
	&\leq 
	\sqrt{\E\sum_{(t,i):i\in\mathcal C_t} p_t(i)p_t(i,S_k)^2 f''(p_t(i,S_k)) \frac{\left(\bar\ell_t(i)-\bar\ell_t(i,S_k)\right)_+^2}{\bar\ell_t(i)} }
	\sqrt{\E\sum_{(t,i):i\in\mathcal C_t} \frac{\bar\ell_t(i)}{p_t(i)f''(p_t(i,S_k))} }
	\\
\nonumber
	&\leq 
	\sqrt{\E\sum_{(t,i)\in [T]\times [d]} p_t(i)p_t(i,S_k)^2 f''(p_t(i,S_k)) \frac{\left(\bar\ell_t(i)-\bar\ell_t(i,S_k)\right)_+^2}{\bar\ell_t(i)} }
	\sqrt{\E\sum_{(t,i):i\in\mathcal C_t} \frac{\bar\ell_t(i)}{p_t(i)f''(p_t(i,S_k))} }
	\\
\nonumber
	&\stackrel{Prop. \ref{prop:TSandMD}}{\leq} 
	\sqrt{2\sum_t\E_t[F(p_{t+1}(\cdot,S_k))-F(p_{t}(\cdot,S_k))]}\cdot\sqrt{\E\sum_{(t,i):i\in\mathcal C_t} \frac{\bar\ell_t(i)}{p_t(i)f''(p_t(i,S_k))} }
	\\
\nonumber
	&\leq  
	\sqrt{2\cdot \E \big[F(p_T(\cdot,S_k))-F(p_1(\cdot,S_k))\big]}\cdot\sqrt{\E\sum_{(t,i):i\in\mathcal C_t} \frac{\bar\ell_t(i)}{p_t(i)f''(p_t(i,S_k))} }
	\\
\label{eq:semibandit-MD-bound}
	&\leq  
	\sqrt{2\cdot \diam_{|S_k|}(F)\cdot\E\sum_{(t,i):i\in\mathcal C_t} \frac{\bar\ell_t(i)}{p_t(i)f''(p_t(i,S_k))} }.
\end{align}
Here $p_t(\cdot,S_k)\in [0,1]^d$ is the vector with $i$-th coordinate $p_t(i,S_k)$. This completes the proof.
\end{proof}

We now prove Theorem~\ref{thm:semibanditTS} whose statement we recall for the reader's convenience.

\semibanditTS*

\begin{proof}
Apply Lemma~\ref{lem:shouldhavelabel} with $f(x)=-x^{1/2}$ and $S_k$ the dyadic partition of $[m]$ (recall \eqref{eq:dyadic}) so that $|S_k|\leq 2^k=\min(S_k)$. Here we take 
\[
	(\gamma_1,\gamma_2)=\left(\frac{m\log^2(\oL^*)}{\oL^*},\frac{md\log^2(\oL^*)}{\oL^*}\right).
\]
Moreover 
\[
	f''(x)=\frac{1}{4x^{3/2}},
	\quad\text{ and }\quad
	\diam_j(F)\leq \sqrt{jd}.
\]
The common arm regret in \eqref{eq:rank-order-common-regret} is at most
\begin{align*}
	\mathbb E\left[\sum_{(t,i,k):i\in \mathcal C_t} p_t(i,S_k)(\bar\ell_t(i)-\bar\ell_t(i,S_k))\right] 
	&\leq  
	O\left(\sum_{k=1}^r\sqrt{d^{1/2}\cdot\E\sum_{(t,i):i\in\mathcal C_t} \frac{\bar\ell_t(i)}{p_t(i)f''(p_t(i,S_k))} }\right) 
	\\
	&\leq  
	O\left(\sum_{k=1}^r\sqrt{2^{k/2}d^{1/2}\cdot\E\sum_{(t,i):i\in\mathcal C_t} \frac{p_t(i,S_k)^{3/2}\bar\ell_t(i)}{p_t(i)} }\right). 
\end{align*}
Cauchy--Schwarz and $p_t(i,S_k)\leq p_t(i)$ now imply:
\[
	\E\sum_{(t,i):i\in\mathcal C_t} \frac{p_t(i,S_k)^{3/2}\bar\ell_t(i)}{p_t(i)} \leq \left(\E\sum_{(t,i):i\in\mathcal C_t} \frac{p_t(i,S_k) \bar\ell_t(i)}{p_t(i)}\right)^{1/2} \cdot \left(\E\sum_{(t,i):i\in\mathcal C_t} p_t(i,S_k) \bar\ell_t(i)\right)^{1/2}.
\]

Using $\gamma_1=\frac{m\log^2(\oL^*)}{\oL^*}$ and Lemma~\ref{lem:semibanditTS} (where an extra factor of $d$ comes from summing over all arms) yields:
\[
	\E\sum_{(t,i):i\in\mathcal C_t} \frac{p_t(i,S_k) \bar\ell_t(i)}{p_t(i)} =d\cdot O\left( \frac{\oL^*}{2^k}+\log\left(\frac{1}{\gamma_1}\right)\sqrt{\frac{\oL^*}{2^k\gamma_1}}\right)=O\left(\frac{d\oL^*}{2^k}\right).
\]
It follows by the definitions that:
\begin{align*}
	\E\sum_{(t,i):i\in\mathcal C_t} p_t(i,S_k) \bar\ell_t(i)
	&\leq
	\E\sum_{(t,i)\in [T]\times [d]} p_t(i) \bar\ell_t(i)
	\\
	&=
	\E[L_T].
\end{align*}
Combining and assuming $\E[R_T]\geq 0$, the common arm regret is at most:
\begin{align*}
	\mathbb E\left[\sum_{(t,i,k):i\in \mathcal C_t} p_t(i,S_k)(\bar\ell_t(i)-\bar\ell_t(i,S_k))\right] 
	&\leq
	O\left(\sum_{k=1}^r \sqrt{2^{k/2}d^{1/2}\cdot \sqrt{\frac{dL^*}{2^k}\cdot \E[L_T]}}\right)
	\\
	&=O\left(\sum_{k=1}^r \sqrt{d(\oL^*+\E[R_T])}\right)
	\\ 
	&= O\left(\log(m)\sqrt{d(\oL^*+\E[R_T])}\right).
\end{align*}
Using the bound \eqref{eq:rank-order-rare-regret} for the rare arm regret and combining, we find
\[
	\E[R_T]\leq O\left(md^2\log^2(\oL^*)+d\log(T)+\log(m)\sqrt{d\cdot(\oL^*+\E[R_T])}\right).
\]
To finish we apply Lemma~\ref{lem:selfbounding} with:
\begin{itemize}
	\item $R=\mathbb E[R_T]$
	\item $X=O(md^2\log^2(\oL^*)+d\log(T))$
	\item $Y=O(\log(m)\sqrt{d})$
	\item $Z=\oL^*$
\end{itemize}
The result is as claimed:
\[
	\E[R_T]\leq O\left(\log(m)\sqrt{d\oL^*}+md^2\log^2(\oL^*)+d\log(T)\right).
\]
\end{proof}

\subsection{Semi-bandit Regret Bound from Log Barrier}

\begin{restatable}{theorem}{semibanditlogbarrier}
\label{thm:semibanditlogbarrier}
Thompson sampling with semi-bandit feedback obeys the regret estimate
\[
	\E[R_T]\leq O\left(\sqrt{d\,\E[L^*]\log(T)}+d\log(T)\right).
\]
\end{restatable}

\begin{proof}
We apply Lemma~\ref{lem:shouldhavelabel} with $f(x)=-\log(Tx+1)$ and $(\gamma_1,\gamma_2)=(\frac{1}{T},\frac{d}{T})$ with no partitioning scheme, i.e. $S_1=[m]$. Then 
\[
	f''(x)^{-1}=(x+T^{-1})^2\leq x^2+\frac{3x}{T}
\]
for $x\geq T^{-1}$. Moreover 
\[
	\diam(F)\leq d\log(T+1)=O(d\log(T)).
\]
Therefore by \eqref{eq:rank-order-common-regret}, the common arm regret is at most
\begin{align*}
	\mathbb E\left[\sum_{(t,i):i\in \mathcal C_t} p_t(i)(\bar\ell_t(i)-\bar\ell_t(i,i))\right] 
	&\leq 
	O\left(\sqrt{ {d\log(T)}\cdot\E\sum_{(t,i):i\in\mathcal C_t} \frac{\bar\ell_t(i)}{p_t(i)f''(p_t(i))} }\right) 
	\\
	&\leq  
	O\left(\sqrt{ d\log(T)\cdot\E\sum_{(t,i):i\in\mathcal C_t} (p_t(i)+3T^{-1})\bar\ell_t(i) }\right)
	\\
	&\leq  
	O\left(\sqrt{d\log(T)\cdot (\E[L_T]+3d) }\right)
	\\
	&\leq  
	O\left(\sqrt{d\log(T)\cdot \E[L_T]}+d\sqrt{\log(T)}) \right).
 \end{align*}
The rare arm regret from \eqref{eq:rank-order-rare-regret} is at most $\gamma_2 T=d$; this is absorbed into the $O(d\sqrt{\log(T)})$ term. In light of \eqref{eq:rank-order-total-regret}, we have established exactly the same estimate as \eqref{eq:log-barrier-almost} in the proof of in Theorem~\ref{thm:MDlogbarrier}. The conclusion follows verbatim.
\end{proof}

\section{Thresholded Thompson Sampling}
\label{sec:threshold}

Unlike in the full-feedback case, our first-order regret bound for bandit Thompson Sampling has an additive $O(d\log(T))$ term. Thus, even when an upper bound $L^*\leq\oL^*$ is known, the regret is $T$-dependent. In fact, some mild $T$-dependence is inherent for any $o(L^*)$ regret bound as shown later in Theorem~\ref{thm:Tdependent}.

 However, this mild $T$-dependence can be avoided by using \emph{Thresholded Thompson Sampling}. In Thresholded Thompson Sampling, the rare arms are \emph{never} played, and the probabilities for the other arms are scaled up correspondingly. In the bandit setting for $\gamma<\frac{1}{d}$, the $\gamma$-thresholded Thompson Sampling algorithm is defined by letting $\mathcal R_t=\{i:p_t(i)\leq \gamma\}$ and playing at time $t$ from the distribution
\[
	\hat p_t(i)=
	\begin{cases}
	 0 & \text{if $i \in \mathcal R_t$} \\
	    \frac{p_t(i)}{1-\sum_{j\in \mathcal R_t}p_t(j)} & \text{if $i\in \mathcal C_t$.}
	\end{cases}
\]
In the combinatorial semi-bandit setting, the corresponding definition is as follows. Set
\begin{equation}
\label{eq:eta-t}
	\eta_t=\sum_{\substack{A'\in\mathcal A:\\A' \not\subseteq \mathcal C_t}}p_t(A') \stackrel{\eqref{eq:semibandit-partition-Rt}}{\leq} d\gamma.
\end{equation}
Then we set
\begin{equation}
\label{eq:thresholded-TS-def}
	\hat p_t(A_t=A)=
	\begin{cases}
	 0 & \text{if $A \not\subseteq \mathcal C_t$} \\
	    \frac{p_t(A)}{1-\eta_t} & \text{if $A \subseteq \mathcal C_t$.}
	\end{cases}
\end{equation}
The key point is that Thresholded Thompson sampling plays arm $i$ with probability either at least $\gamma$ (if $i\in\mathcal C_t$) or $0$ (if $i\in\mathcal R_t$).

This algorithm parallels the work \cite{LST18} which uses an analogous modification of the EXP3 algorithm to obtain a first-order regret bound. Note that in the semi-bandit setting, for $i\in\mathcal C_t$ it may be that $\hat p_t(i)<p_t(i)$. However $\hat p_t(i)\geq \gamma$ always holds, ensuring that Theorem~\ref{thm:bayesianexplorer} applies.

We first give our main guarantee for Thresholded Thompson sampling in the bandit case with $m=1$, which is based on Tsallis entropy. The result below could be slightly refined by incorporating the Tsallis entropy $H_{\alpha}(p_1)$ into the regret estimate as in Theorem~\ref{thm:MDtsallis}, but we have instead elected for simplicity in the statement. The analysis works also with Shannon entropy (which again gives a slightly weaker bound), but seemingly not with the log barrier.

\begin{restatable}{theorem}{thresholdedbanditTS}
\label{thm:thresholdedbanditTS}

 Suppose that $L^*\leq \oL^*$ holds almost surely for a constant $\oL^*$. Thompson Sampling for bandit feedback, thresholded with $\gamma=\frac{\log^2(\oL^*)}{\oL^*}\leq \frac{1}{2d}$, has expected regret 
\[
	\E[R_T] = O\left(\sqrt{d\oL^*}+d\log^2(\oL^*)\right).
\]
\end{restatable}

\begin{proof}
For any $t,i\in [T]\times [d]$ it holds from \eqref{eq:thresholded-TS-def} and \eqref{eq:eta-t} that 
\begin{equation}
\label{eq:hatpt-pt}
	\hat p_t(i)\leq \frac{p_t(i)}{1-\eta_t}\leq \frac{p_t(i)}{1-\gamma d}.
\end{equation}
We again apply Proposition~\ref{prop:bandit-regret-basic}, this time in the general setting which allows $\hat p_t\neq p_t$. The result is:
\begin{align*}
	\E[R_T]
	&=
	\E\left[\sum_{(t,i)\in [T]\times [d]} \hat p_t(i)\bar\ell_t(i)-p_t(i)\bar\ell_t(i,i)\right]
	\\
	&=
	\E\left[(\hat p_t(i)-p_t(i))\bar\ell_t(i,i)\right] + \E\left[\hat p_t(i)(\bar\ell_t(i)-\bar\ell_t(i,i)\right]
	\\
	&\leq 
	\left(\frac{\gamma d}{1-\gamma d}\right)\cdot \mathbb E\left[\sum_{(t,i)\in [T]\times [d]} p_t(i)\bar\ell_t(i,i)\right] + \mathbb E\left[\sum_{(t,i):i\in\mathcal C_t} \hat p_t(i)(\bar\ell_t(i)-\bar\ell_t(i,i))\right]
	\\
	&\leq
	2\gamma d\cdot \mathbb E\left[\sum_{(t,i)\in [T]\times [d]} p_t(i)\bar\ell_t(i,i)\right] + \mathbb E\left[\sum_{(t,i):i\in\mathcal C_t} \hat p_t(i)(\bar\ell_t(i)-\bar\ell_t(i,i))\right].
\end{align*}
Here the last step follows from the assumption $\gamma\leq\frac{1}{2d}$. 
The former expectation is 
\begin{align*}
	2\gamma d\cdot\mathbb E\left[\sum_{(t,i)\in [T]\times [d]} p_t(i)\bar\ell_t(i,i)\right]
	&=
	2\gamma d\cdot\mathbb E\left[\sum_{t=1}^T \ell_t(a^*)\right]
	\\
	&\leq 
	2\gamma d\cdot\oL^*
	\\
	&\leq
	O(d \log^2(\oL^*)).
\end{align*}
The latter can be bounded in the same way as the non-thresholded results. Intuitively, since \eqref{eq:hatpt-pt} implies
\begin{equation}
\label{eq:thresholded-approx}
	\hat p_t(i)\leq 2p_t(i)
	\quad\quad
	\forall i\in\mathcal C_t,
\end{equation}
the calculation should be almost the same. 
To make this precise we imitate \eqref{eq:bandit-MD-bound} (which was the same calculation but with $\hat p_t=p_t$). The result is:
\begin{equation}
\label{eq:bandit-thresholded-MD-bound}
\begin{aligned}
	\E&\left[\sum_{(t,i):i\in \mathcal C_t}  \hat p_t(i)\cdot (\bar\ell_t(i)-\bar\ell_t(i,i))_+\right]
	\\
	&\leq 
	\sqrt{
		\E\sum_{(t,i):i\in\mathcal C_t} \hat p_t(i) p_t(i)^2 f''(p_t(i)) 
		\frac{\left(\bar\ell_t(i)-\bar\ell_t(i,i)\right)_+^2}
		{\bar\ell_t(i)} 
	}
	\cdot
	\sqrt{
		\E\sum_{(t,i):i\in\mathcal C_t} 
		\frac{\hat p_t(i)\bar\ell_t(i)}
		{p_t(i)^2 f''(p_t(i))} 
	}
	\\
	&\leq 
	\sqrt{\E\sum_{(t,i)\in [T]\times [d]} \hat p_t(i) p_t(i)^2 f''(p_t(i)) \frac{\left(\bar\ell_t(i)-\bar\ell_t(i,i)\right)_+^2}{\bar\ell_t(i)} }\cdot\sqrt{\E\sum_{(t,i):i\in\mathcal C_t} 
	\frac{\hat p_t(i)\bar\ell_t(i)}
		{p_t(i)^2 f''(p_t(i))} }
	\\
	&\stackrel{Prop. \ref{prop:TSandMD}}{\leq}
	\sqrt{2\sum_{t=1}^T\E_t[F(p_{t+1})-F(p_{t})]}\cdot\sqrt{\E\sum_{(t,i):i\in\mathcal C_t} \frac{\hat p_t(i)\bar\ell_t(i)}
		{p_t(i)^2 f''(p_t(i))}  }
	\\
	&\leq  
	\sqrt{2\cdot \E[F(p_T)-F(p_1)]}\cdot\sqrt{\E\sum_{(t,i):i\in\mathcal C_t} \frac{\hat p_t(i)\bar\ell_t(i)}
		{p_t(i)^2 f''(p_t(i))}  }
	\\
	&\leq  \sqrt{2\cdot (\Max(F)-F(p_1))\cdot\E\sum_{(t,i):i\in\mathcal C_t} \frac{\hat p_t(i)\bar\ell_t(i)}
		{p_t(i)^2 f''(p_t(i))}  }.
\end{aligned}
\end{equation}
All justifications are identical to \eqref{eq:bandit-MD-bound} (which is the special case $\hat p_t=p_t$ of the above). We complete the estimation using Tsallis entropy as in Theorem~\ref{thm:MDtsallis}. Set $f(x)=-x^{\alpha}$ so that
\begin{align*}
	f''(x)&=\alpha(1-\alpha)x^{\alpha-2}
	;
	\\
	\Max(F)&=-1;
	\\
	\Min(F)&=-d^{1-\alpha}.
\end{align*}
By \eqref{eq:thresholded-approx}, we then have (for $c_{\alpha},c'_{\alpha}$ constants depending on $\alpha$):
\begin{align*}
	\frac{\hat p_t(i)\bar\ell_t(i)}
		{p_t(i)^2 f''(p_t(i))} 
	&=
	c_{\alpha}\cdot\left(\frac{\hat p_t(i)\bar\ell_t(i)}
		{p_t(i)^{\alpha}}\right) 
	\\
	&\leq
	c'_{\alpha}\hat p_t(i)^{1-\alpha}\bar\ell_t(i).
\end{align*}
Then \eqref{eq:bandit-thresholded-MD-bound} specializes to
\begin{align*}
	\E\left[\sum_{(t,i):i\in \mathcal C_t}  \hat p_t(i)\cdot (\bar\ell_t(i)-\bar\ell_t(i,i))_+\right]
	&\leq
	O(1)\cdot\sqrt{d^{1-\alpha}\cdot\E\sum_{(t,i):i\in\mathcal C_t} \frac{\hat p_t(i)\bar\ell_t(i)}{p_t(i)^2f''(p_t(i))} }
	\\
	&\leq
	O_{\alpha}(1)\cdot\sqrt{d^{1-\alpha}\cdot
	\E\sum_{(t,i):i\in\mathcal C_t} \bar\ell_t(i)\hat p_t(i)^{1-\alpha}}.
\end{align*}
Applying H{\"o}lder's inequality in the first step, we find
\begin{equation}
\label{eq:bandit-thresholded-calc}
\begin{aligned}
	\mathbb E\sum_{(t,i):i\in\mathcal C_t} \bar\ell_t(i)\hat p_t(i)^{1-\alpha}
	&\leq 
	\left(
		\mathbb E\sum_{(t,i):i\in\mathcal C_t} \bar\ell_t(i)
	\right)^{\alpha}
	\left(
		\mathbb E\sum_{(t,i):i\in\mathcal C_t}\bar\ell_t(i)\hat p_t(i)
	\right)^{1-\alpha}
	\\
	&\leq  
	\left( 
		\oL^*+2\left(\log\left(\frac{1}{\gamma}\right)+10\right)\sqrt{\frac{\oL^*}{\gamma}} 
	\right)^{\alpha}
	\cdot \mathbb E[L_T]^{1-\alpha}
	\\
	&\leq
	O(d\oL^*)^{\alpha}\cdot\mathbb E[L_T]^{1-\alpha}.
\end{aligned}
\end{equation}
In the second step of \eqref{eq:bandit-thresholded-calc}, the first term is bounded as usual by Theorem~\ref{thm:bayesianexplorer}. 
Paralleling \eqref{eq:loss-formula}, the second term is bounded by observing 
\begin{align*}
	\mathbb E\left[\sum_{(t,i):i\in\mathcal C_t}\bar\ell_t(i)\hat p_t(i)\right]
	&\leq
	\mathbb E\left[\sum_{(t,i)\in [T]\times [d]}\bar\ell_t(i)\hat p_t(i)\right]
	\\
	&=
	\mathbb E[L_T].
\end{align*}
The last step in \eqref{eq:bandit-thresholded-calc} again follows from the choice of $\gamma$ which ensures
\[
	\oL^*+2\left(\log\left(\frac{1}{\gamma}\right)+10\right)\sqrt{\frac{\oL^*}{\gamma}}\leq O(\oL^*).
\]
Assuming $\E[R_T]\geq 0$ and combining the above calculations, we find 
\[
	\E[R_T]
	\leq
	O_{\alpha}\left(
	d^2\log(\oL^*) + 
	\sqrt{d\left(\oL^*+\E[R_T]\right)}
	\right).
\]
Applying Lemma~\ref{lem:selfbounding} as in Theorem~\ref{thm:MDtsallis} (but without the $\log(T)$ term) and choosing arbitrary $\alpha\in (0,1)$ completes the proof.
\end{proof}

In the semibandit setting, our previous analysis is similarly adapted.

 \begin{restatable}{theorem}{thresholdedsemibanditTS}
 \label{thm:thresholdedsemibanditTS}

 Suppose that the best combinatorial action almost surely has total loss at most $\oL^*$. Thompson Sampling for semi-bandit feedback, thresholded with $\gamma=\frac{m\log^2(\oL^*)}{\oL^*}\leq \frac{1}{2d}$, has expected regret 
 \[
 	\E[R_T]=O\left(\log(m)\sqrt{d\oL^*}+md\log^2(\oL^*)\right).
 \]
 \end{restatable}

\begin{proof}
Thresholding at $\gamma$ removes at most $d\gamma$ total probability of actions, so as before $\hat p_t(i)\leq \frac{p_t(i)}{1-\gamma d}.$ The start of the calculation (this time using Proposition~\ref{prop:semibandit-regret-basic}) goes
\begin{equation}
\label{eq:semibandit-thresholded-TS-regret-decomp}
\begin{aligned}
	\E[R_T]&=\E\left[\sum_{(t,i)\in [T]\times [d]} \hat p_t(i)\bar\ell_t(i)-p_t(i)\bar\ell_t(i,i)\right]
	\\
	&=
	\E\left[(\hat p_t(i)-p_t(i))\bar\ell_t(i,i)\right] + \E\left[\hat p_t(i)(\bar\ell_t(i)-\bar\ell_t(i,i)\right]
	\\
	&\leq 
	\left(\frac{\gamma d}{1-\gamma d}\right)\cdot \mathbb E\left[\sum_{(t,i)\in [T]\times [d]} p_t(i)\bar\ell_t(i,i)\right] + \mathbb E\left[\sum_{(t,i):i\in\mathcal C_t} \hat p_t(i)(\bar\ell_t(i)-\bar\ell_t(i,i))\right]
	\\
	&\leq 
	2\gamma d\oL^*+ \sum_{k=1}^r\mathbb E
	\left[
		\sum_{(t,i):i\in\mathcal C_t} 
		\frac{\hat p_t(i)p_t(i,S_k)}{p_t(i)}
		(\bar\ell_t(i)-\bar\ell_t(i,S_k))
	\right].
\end{aligned}
\end{equation}
Take the sets $S_k$ as in \eqref{eq:dyadic}, the dyadic partition of $[m]$, so that $|S_k|\leq 2^k=\min(S_k)$. Thresholding at $\gamma=\frac{m\log^2(\oL^*)}{\oL^*}$, the first term above is 
\begin{equation}
\label{eq:semibandit-thresholded-TS-rare}
	2\gamma d\oL^*\leq 2md\log^2(\oL^*).
\end{equation}

To control the main sum involving $\mathcal C_t$, we combine the analyses of Lemma~\ref{lem:shouldhavelabel} and Theorem \ref{thm:thresholdedbanditTS}. For each $k\in [r]$, similarly to \eqref{eq:bandit-MD-bound}, \eqref{eq:semibandit-MD-bound}, and \eqref{eq:bandit-thresholded-MD-bound} we obtain:
\begin{align}
\nonumber
	\E&\left[\sum_{(t,i):i\in \mathcal C_t}  \frac{\hat p_t(i)p_t(i,S_k)}{p_t(i)}
	\cdot 
	(\bar\ell_t(i)-\bar\ell_t(i,S_k))_+\right]
	\\
\nonumber
	&\leq 
	\sqrt{
		\E\sum_{(t,i):i\in\mathcal C_t} \hat p_t(i) p_t(i,S_k)^2 f''(p_t(i,S_k)) 
		\frac{\left(\bar\ell_t(i)-\bar\ell_t(i,S_k)\right)_+^2}
		{\bar\ell_t(i)} 
	}
	\cdot
	\sqrt{
		\E\sum_{(t,i):i\in\mathcal C_t} 
		\frac{\hat p_t(i)\bar\ell_t(i)}
		{p_t(i)^2 f''(p_t(i,S_k))} 
	}
	\\
\nonumber
	&\leq 
	\sqrt{
		\E\sum_{(t,i)\in[T]\times [d]} \hat p_t(i) p_t(i,S_k)^2 f''(p_t(i,S_k)) 
		\frac{\left(\bar\ell_t(i)-\bar\ell_t(i,S_k)\right)_+^2}
		{\bar\ell_t(i)} 
	}
	\cdot
	\sqrt{
		\E\sum_{(t,i):i\in\mathcal C_t} 
		\frac{\hat p_t(i)\bar\ell_t(i)}
		{p_t(i)^2 f''(p_t(i,S_k))} 
	}
	\\
\nonumber
	&\stackrel{Prop. \ref{prop:TSandMD}}{\leq}
	\sqrt{2\sum_{t=1}^T\E_t\big[F(p_{t+1}(\cdot,S_k))-F(p_{t}(\cdot,S_k))\big]}
	\cdot
	\sqrt{\E\sum_{(t,i):i\in\mathcal C_t} \frac{\hat p_t(i)\bar\ell_t(i)}
		{p_t(i)^2 f''(p_t(i,S_k))}  }
	\\
\nonumber
	&\leq  
	\sqrt{2\cdot \E[F(p_T(\cdot,S_k))-F(p_1(\cdot,S_k))]}\cdot\sqrt{\E\sum_{(t,i):i\in\mathcal C_t} \frac{\hat p_t(i)\bar\ell_t(i)}
		{p_t(i)^2 f''(p_t(i,S_k))}  }
	\\
\label{eq:semibandit-thresholded-MD-bound}
	&\leq  \sqrt{2\cdot (\Max_{|S_k|}(F)-F(p_1(\cdot,S_k)))\cdot\E\sum_{(t,i):i\in\mathcal C_t} \frac{\hat p_t(i)\bar\ell_t(i)}
		{p_t(i)^2 f''(p_t(i,S_k))}  }.
\end{align}
Here $p_t(\cdot,S_k)\in [0,1]^d$ is the vector with $i$-th coordinate $p_t(i,S_k)$. We take $f(x)=-x^{1/2}$  so that
\begin{equation}
\label{eq:tsallis-final-facts}
	f''(x)=\frac{1}{4x^{3/2}},
	\quad\text{ and }\quad
	\diam_j(F)\leq \sqrt{jd}.
\end{equation}
We continue from \eqref{eq:semibandit-thresholded-MD-bound}, now summing over $k\in [r]$. Recall that $|S_k|\leq 2^k=\min(S_k)=s_k$ and 
\begin{equation}
\label{eq:pt-ineq}
	\max\big(p_t(i,S_k),\hat p_t(i)\big)
	\leq 2 p_t(i).
\end{equation}
We find:
\begin{equation}
\label{eq:no-more-names-help}
\begin{aligned}
	\sum_{k=1}^r 
	\E&\left[\sum_{(t,i):i\in \mathcal C_t}  \frac{\hat p_t(i)p_t(i,S_k)}{p_t(i)}
	\cdot 
	(\bar\ell_t(i)-\bar\ell_t(i,S_k))_+\right]
	\\
	&\stackrel{\eqref{eq:semibandit-thresholded-MD-bound}, \eqref{eq:tsallis-final-facts}}{\leq}
	O\left(
		\sum_{k=1}^r
		\sqrt{\sqrt{|S_k|d}\cdot\E\sum_{(t,i):i\in\mathcal C_t} \frac{\hat p_t(i) p_t(i,S_k)^{3/2}\bar\ell_t(i)}{p_t(i)^2} }
	\right)
	\\
	&\stackrel{\eqref{eq:pt-ineq}}{\leq}
	O\left(
	\sum_{k=1}^r
	\sqrt{
		\sqrt{|S_k|d}\cdot\E
		\sum_{(t,i):i\in\mathcal C_t} 
		\frac{\hat p_t(i)^{1/2} p_t(i,S_k)^{1/2} \bar\ell_t(i)}{p_t(i)^{1/2}} 
	}
	\right)
	\\
	&\leq
	O\left(
	\sum_{k=1}^r
	\left(
		|S_k|d\cdot
		\E\left[
			\sum_{(t,i):i\in\mathcal C_t}
			\hat p_t(i) \bar\ell_t(i)
		\right]
		\E\left[
			\sum_{(t,i):i\in\mathcal C_t}
			\frac{p_t(i,S_k) \bar\ell_t(i)}{p_t(i)}  
		\right]
	\right)^{1/4}
	\right).
\end{aligned}
\end{equation}
By definition, the first inner sum is bounded by
\begin{align*}
	\E\left[
			\sum_{(t,i):i\in\mathcal C_t}
			\hat p_t(i) \bar\ell_t(i)
		\right]
	&\leq 
	\E\left[
			\sum_{(t,i)\in [T]\times [d]}
			\hat p_t(i) \bar\ell_t(i)
		\right]
	\\
	&=\E[L_T]
\end{align*}
Using Lemma~\ref{lem:semibanditTS} for each $i\in [d]$ and then the definition of $\gamma$, we obtain 
\begin{align*}
	\E\left[
			\sum_{(t,i):i\in\mathcal C_t}
			\frac{p_t(i,S_k) \bar\ell_t(i)}{p_t(i)}  
		\right]
	&\leq 
	d\cdot O\left(\frac{\oL^*}{s_k}+\log\left(\frac{1}{\gamma}\right)\sqrt{\frac{\oL^*}{s_k\gamma}}\right)
	\\
	&\leq
	O\left(\frac{d\oL^*}{s_k}\right).
\end{align*}
Substituting the previous two displays into \eqref{eq:no-more-names-help} and assuming $\E[R_T]\geq 0$, we find
\begin{align*}
	\sum_{k=1}^r 
	\E\left[\sum_{(t,i):i\in \mathcal C_t}  \frac{\hat p_t(i)p_t(i,S_k)}{p_t(i)}
	\cdot 
	(\bar\ell_t(i)-\bar\ell_t(i,S_k))_+\right]
	&\leq
	O\left(
	\sum_{k=1}^r
	\left(
		|S_k|d\cdot
		\E[L_T]\cdot
		\frac{d\oL^*}{s_k}
	\right)^{1/4}
	\right)\\
	&\leq
	O\left(
	\sum_{k=1}^r
	\sqrt{
		d(\oL^*+\E[R_T])
	}
	\right)
	\\
	&\leq
	O\left(
	\log(m) 
	\sqrt{
		d(\oL^*+\E[R_T])
	}
	\right)
	.
\end{align*}
Combining with \eqref{eq:semibandit-thresholded-TS-regret-decomp} and \eqref{eq:semibandit-thresholded-TS-rare} we conclude that 
\[
	\E[R_T]
	\leq
	O\left(md\log^2(\oL^*) + \log(m) 
	\sqrt{
		d(\oL^*+\E[R_T])
	}
	\right).
\]
The proof is now concluded via Lemma~\ref{lem:selfbounding} similarly to the end of proving Theorem~\ref{thm:semibanditTS}.
\end{proof}

\section{Graphical Feedback}

We now consider online learning with graphical feedback. This model interpolates between full-feedback and bandits by embedding the actions as vertices of a (possibly directed) feedback graph $G$.
Here playing action $a_t=i$ allows one to observe the losses $\ell_t(j)$ for all $j$ such that an edge $i\to j$ exists in $G$. We assume that all vertices $i\in [d]$ have self-loops $i\to i$, i.e. that we always observe the loss incurred by the action played. Without this assumption, the optimal regret can be $\tilde\Theta(T^{2/3})$ even if every vertex is observable, see \cite{alon2015online}.

Previous work such as \cite{liu2018analysis,tossou2017thompson} analyzed the performance of Thompson Sampling for these tasks, giving $O(\sqrt{T})$-type regret bounds which scale with certain statistics of the graph. However, their analyses only applied for stochastic losses rather than adversarial losses. In this section, we outline why their analysis applies to the adversarial case as well.

Let $G$ be a possibly directed feedback graph on $d$ vertices, with $\alpha=\alpha(G)$ the size of its maximum independent set. We use the following lemma:

\begin{lemma}[\cite{mannor2011bandits}, Lemma 3]
For any probability distribution $\pi$ on $V(G)$ (with the convention $0/0=0$):
\[
	\sum_{i=1}^d 
	\frac{\pi(i)}{\sum_{j\in \{i\}\cup N(i)} \pi(j)} \leq \alpha.
\] 
\end{lemma}

Following \cite{liu2018analysis} we now obtain:

\begin{proposition} 

The coordinate information ratio of Thompson Sampling on an undirected graph $G$ is at most $\alpha(G)$.

\end{proposition}

\begin{proof}
Let $p_t(i)$ be as usual for a vertex $i$ and $q_t(i)=\sum_{j\in\{i\}\cup N(i)} p_t(i)$ the probability to observe $\ell_t(i)$. Then:
\[
	\alpha\cdot I_t^c\geq \left(\sum_{i=1}^d \frac{p_t(i)}{q_t(i)}\right)\left(\sum_{i=1}^d p(i)q(i)(\ell_t(i)-\ell_t(i,i))^2\right)\geq R^2.
\]
\end{proof}

In the case of a directed graph, a natural analog of $\alpha(G)$ is the maximum value of 
\[
	\sum_{i=1}^d \frac{\pi(i)}{\sum_{j\in \{i\}\cup N^{in}(i)} \pi(j)}
\]
which is equal to $\mathrm{mas}(G)$, the size of the maximal acyclic subgraph of $G$. However, as noted in \cite{liu2018analysis}, if we assume 
\[
	\pi_t(i)\geq \varepsilon
\]
for all $(t,i)\in [T]\times [d]$, then \cite{alon2015online} gives the upper bound 
\begin{equation}
\label{eq:digraph-eps}
	\sum_{i=1}^d \frac{\pi(i)}{\sum_{j\in \{i\}\cup N^{in}(i)} \pi(j)}\leq 4\left(\alpha\cdot \log\left(\frac{4d}{\alpha\varepsilon}\right)\right).
\end{equation}
Of course, $\varepsilon=(dT)^{-3}$ additional exploration has essentially no effect on the expected regret (as it induces $O(T^{-2})$ total variation distance betwen the two algorithms and hence adds $O(1/T)$ regret). By mixing Thompson sampling with an $\epsilon=(dT)^{-3}$ probability of uniform exploration at each time, the bound \eqref{eq:digraph-eps} thus applies and we obtain a $\alpha$-dependent bound for directed graphs as well.

\begin{theorem}
Thompson Sampling on a sequence $G_t$ of undirected graphs achieves expected regret
\[
	\mathbb E[R_T]=O\left(\sqrt{H^c(p_1)\sum_{t=1}^T \alpha(G_t)}\right).
\]
Moreover Thompson Sampling on a sequence $G_t$ of directed graphs achieves expected regret
\[
	\mathbb E[R_T]=O\left(\sqrt{H^c(p_1)\log(dT)\sum_{t=1}^T \alpha(G_t)}\right).
\]
\end{theorem}

As in \cite{liu2018analysis}, this analysis applies even when the Thompson sampling algorithm does not know the graphs $G_t$, but only observes the relevant neighborhood feedback after choosing each action $a_t$.

\section{Negative Results for Thompson Sampling}

Here we present some negative results. First, Theorem~\ref{thm:nohighprob} states that Thompson Sampling against an arbitrary prior may have $\Omega(T)$ regret a constant fraction of the time (but will therefore also have $-\Omega(T)$ regret a constant fraction of the time). By contrast, there exist algorithms which have low regret with high probability even in the frequentist setting \cite{exp3p}. Bridging this gap with a variant of Thompson Sampling would be very interesting. 

\begin{restatable}{theorem}{nohighprob}
\label{thm:nohighprob}

For all $T\geq T_0$ at least an absolute constant, there exists a prior distribution on $d=2$ arms for which Thompson Sampling incurs at least $\frac{T}{3}$ regret with probability at least $\frac{1}{3}$ (with either full or bandit feedback).

 \end{restatable}

\begin{proof}
We construct such a prior distribution with $2$ arms. First for $t\leq T/3$ we take $\ell_t(1)=1$ and $\ell_t(2)=0$ almost surely. Afterward exactly one of the following two possibilities occurs, each with probability $\frac{1}{2}$.
\begin{enumerate}
	\item For $t>T/3$, we have $\ell_t(1)=\ell_t(2)=0$.
	\item For $t>T/3$, we have $\ell_t(1)=0$ and $\ell_t(2)=1$.
\end{enumerate}

In this construction, Thompson Sampling will pick arm $1$ with probability $\frac{1}{2}$ during each of the first $T/3$ rounds. Hence there is an $1-o_{T\to\infty}(1)$ probability to have $L_T\geq \frac{T}{3}$. On the other hand, $L^*=0$ with probability $\frac{1}{2}$ from the first case above. Therefore $R_T\geq \frac{T}{3}$ with probability $\frac{1}{2}-o_{T\to\infty}(1)$. This completes the proof.
\end{proof}

Recall that even in Theorem~\ref{thm:MDtsallis} there was an additive $d\log(T)$ term in the expected regret. Of course, once the player incurs loss $\oL^*+1$ on arm $i$, Thompson sampling will never play arm $i$ again. Therefore the total loss for Thompson sampling (ordinary or Thresholded) can never be more than $d(\oL^*+1)$. Theorem~\ref{thm:MDtsallis} leaves open the possibility that $\Omega(d\oL^*)$ regret is eventually reached when $T$ is extremely large. In other words, our regret bound for ordinary Thompson sampling becomes trivial for extremely large $T$ when $d$ and $\oL^*$ are fixed. Theorem~\ref{thm:Tdependent} below shows that this reflects reality. Namely, there do exist prior distributions for which $\Omega(d\oL^*)$ expected regret is incurred by Thompson sampling for large $T$.

\begin{restatable}{theorem}{Tdependent}
\label{thm:Tdependent}

Let $d\geq 3$. There exist prior distributions against which Thompson Sampling achieves $\Omega(d\oL^*)$ expected regret for very large $T$ with bandit feedback, even given the value $\oL^*$.

\end{restatable}

\begin{proof}
We construct such a prior distribution on $d\geq 3$ arms is as follows. First pick a uniformly random ``good" arm $a^*\in [d]$. For $i\in [d]\backslash\{a^*\}$, set arm $i$ to be either ``bad'' or ``terrible'' uniformly at random, independently over different arms $i$. Denote by $\cB$ and $\cT$ the sets of bad and terrible arms, respectively.

The (random) loss sequence $(\ell_t(i))_{(t,i)\in [T]\times [d]}$ is constructed as follows. First at time $i$, we set
\[
	\ell_t(i)=1_{i\neq a^*},
	\quad\quad 
	i\in [d].
\]
In other words, all arms except $a^*$ receive a loss. Next for $a^*$, every subsequent loss $\ell_t(a^*)$ is uniformly random in $\{0,1\}$ until the first time $\tau$ with total loss $L_{\tau}(a^*)=\oL^*$ is reached. For $t\geq \tau$, we set $\ell_t(a^*)=0$.

For each bad arm $i\in\cB$, we do the same with $\ell_t(i)$ uniformly random in $\{0,1\}$ for $t>1$, but stop at total loss $\oL^*+1$ instead of $\oL^*$.

For a terrible arm $i\in\cT$, we let the losses $\ell_t(i)\in \{0,1\}$ for $t>1$ be uniformly random for all time (so e.g. the total loss grows linearly with $T$).

If $a_1=a^*$, then Thompson sampling will observe $\ell_t(a_1)=0$ and thus infer that $a^*=a_1$. Hence in this case we have $a_t=a^*$ for all $t\geq 1$ and there will be no regret.
However, suppose that $a_1\neq a^*$, which holds with probability $\frac{d-1}{d}$. We claim that on this event, the player will pay loss $\oL^*+1$ on each terrible arm with probability $1-o(1)$ for sufficiently large $T$. This implies the desired result.

Indeed, suppose $i\in\cT$ satisfies $i\neq a_1$ was not played at time $1$. Fix a time $t$ and let $\alpha_i(t)<t$ be the most recent time that $a_{\hat t}=i$ was played. Moreover suppose that $L_t(i)<\oL^*$. Then we claim that $p_t(i)$ is uniformly bounded away from $0$ until the value $\alpha_i(t)$ changes, i.e. until the next time $s>t$ that $a_s=i$.

To do this we consider the alternative hypothesis for the player which differs from the truth in that $a^*\in\cB$ is actually a bad arm, while $i$ is actually the good arm. The former change only affects the distribution of the sequence $(\ell_t(a^*))_{t\geq 1}$ in the value $\ell_1(a^*)$, which was not observed by assumption. Moreover the player only makes Bayesian updates regarding the latter change when $a_t=i$ is played. Finally this evidence is never conclusive until the player has suffered loss
\[
	\sum_{s\leq t} \ell_s(i)\cdot 1_{a_t=i}>\oL^*.
\] 
It follows that while $\alpha_i(t)$ is constant, the posterior likelihood ratio between this alternative hypothesis and the true arm identities is at least $\eps(\alpha_i(t))>0$.

Additionally, with probability $1$ the player's probability assigned to the true arm configuration is bounded away from $0$ uniformly in time. Indeed that probability is a martingale, and if this were false then the probability would have to converge to $0$. But the player's subjective probability of this (true) statement cannot converge to $0$, because revealing more information (i.e. all losses for all times) would then also assign the true statement probability $0$ by the martingale property, a contradiction.

Since for fixed $\alpha_i(t)$ the Bayes factor between the truth and the alternative is bounded, we see that this alternative with arm $i$ as the good arm has probability at least $\eps'(\alpha_i(t))>0$ not depending explicitly on $t$. 

We have just argued that Thompson Sampling with this prior will have a uniformly positive probability to play such an arm $i$ until the nex time it plays $i$ again. Thus, with probability $\frac{d-1}{d}$ (for the first arm not to be good), Thompson Sampling accumulates loss $\oL^*+1$ on every terrible arm except the first arm it plays when run for an infinite amount of time. By countable exhaustion, the same holds for sufficiently large finite $T$ with loss $\oL^*+1-o(1)\geq \oL^*$. This results in $\Omega(d\oL^*)$ regret since the average number of terrible arms is $\frac{d-1}{2}$.
\end{proof}

Finally we show that Thompson sampling does not achieve good small-loss bounds for contextual bandits. Recall that abstractly, contextual bandit is equivalent to graph feedback in which:
\begin{itemize}
	\item The graphs change from round to round.
	\item All graphs are vertex-disjoint unions of at most $K$ cliques.
	\item The losses for a round are constant within cliques.
\end{itemize}

The existence of an algorithm achieving $O(\sqrt{L^*})$ regret for contextual bandits was asked in \cite{coltopen} and resolved positively in \cite{MYGA} with a computationally intractable algorithm, and later in \cite{foster2021efficient} with an efficient algorithm assuming access to a regression oracle. It would be interesting to design a natural Bayesian algorithm matching these guarantees. 

\begin{restatable}{theorem}{contextuallowerbound}
\label{thm:contextuallowerbound}

There exists a prior distribution on which Thompson Sampling achieves, with high probability, regret $\Omega(\sqrt{T})$ for a contextual bandit problem with $L^*=0$ optimal loss, $K=2$ cliques, and $d=O(\sqrt{T})$ total arms.

\end{restatable}

\begin{proof}
Set $S=\sqrt{T}$ and fix $d\geq 2S$. Form $S$ distinct \emph{small cliques}, with random but disjoint sets of $\frac{d}{2S}$ arms each. Call these cliques $C_1,\dots, C_{S}$. Also generate independent uniformly random bits $b_1,\dots, b_S \in \{0,1\}$. For each $j\in \{0,1,\dots, \sqrt{T}-1\}$, consider the set of times $\cT_j=\{jS+1,\dots, (j+1)S\}\subseteq [T]$.

For $t\in \cT_j$, we set the feedback graph $G_t$ consist of the clique $C_j$ and the complementary clique on $[d]\backslash C_j$. We take the loss on the small clique $C_j$ to be $b_i$, and $0$ on the complement $[d]\backslash C_j$. Finally, at the last time $T$ pick at random a single arm $a^*$ with no loss so far and make the loss
\[
	\ell_t(i)=1_{i\neq a^*}.
\]
(This corresponds to the trivial clique on $a^*$, and the clique on $[d]\backslash\{a^*\}$.) Then clearly $L^*=0$ for arm $a^*$.

However Thompson Sampling will incur a constant expected loss for each clique $C_{j}$. This is because until observing a loss on $C_j$ during $t\in\cT_j$, there is a $\Theta(T^{-1/2})$ probability that $a^*\in C_j$ eventually holds, and there are $|\cT_j|=\Theta(T^{1/2})$ opportunities for Thompson sampling to choose an arm in $C_j$. In all, Thompson sampling incurs expected loss $\Theta(S)=\Theta(\sqrt{T})$ as claimed.
\end{proof}

\bibliographystyle{alpha}
\bibliography{refs}

\newcommand{\etalchar}[1]{$^{#1}$}
\begin{thebibliography}{ACBDK15}

\bibitem[AAGO06]{AAGO06}
C.~Allenberg, P.~Auer, L.~Gy{\"{o}}rfi, and G.~Ottucs{\'{a}}k.
\newblock {H}annan consistency in on-line learning in case of unbounded losses
  under partial monitoring.
\newblock In {\em Proceedings of the 17th International Conference on
  Algorithmic Learning Theory (ALT)}, 2006.

\bibitem[ABL14]{ABL14}
J.Y. Audibert, S.~Bubeck, and G.~Lugosi.
\newblock Regret in online combinatorial optimization.
\newblock {\em Mathematics of Operations Research}, 39:31--45, 2014.

\bibitem[ACBDK15]{alon2015online}
Noga Alon, Nicolo Cesa-Bianchi, Ofer Dekel, and Tomer Koren.
\newblock Online learning with feedback graphs: Beyond bandits.
\newblock In {\em Annual Conference on Learning Theory}, volume~40. Microtome
  Publishing, 2015.

\bibitem[ACBFS02]{exp3p}
Peter Auer, Nicolo Cesa-Bianchi, Yoav Freund, and Robert~E Schapire.
\newblock The nonstochastic multiarmed bandit problem.
\newblock {\em SIAM journal on computing}, 32(1):48--77, 2002.

\bibitem[AG12]{AG12}
S.~Agrawal and N.~Goyal.
\newblock Analysis of {T}hompson sampling for the multi-armed bandit problem.
\newblock In {\em Proceedings of the 25th Annual Conference on Learning Theory
  (COLT)}, JMLR Workshop and Conference Proceedings Volume 23, 2012.

\bibitem[AKL{\etalchar{+}}17]{coltopen}
Alekh Agarwal, Akshay Krishnamurthy, John Langford, Haipeng Luo, et~al.
\newblock Open problem: First-order regret bounds for contextual bandits.
\newblock In {\em Conference on Learning Theory}, pages 4--7, 2017.

\bibitem[AZBL18]{MYGA}
Zeyuan Allen-Zhu, Sebastien Bubeck, and Yuanzhi Li.
\newblock Make the minority great again: First-order regret bound for
  contextual bandits.
\newblock In {\em International Conference on Machine Learning}, pages
  186--194, 2018.

\bibitem[BDKP15]{BDKP15}
S.~Bubeck, O.~Dekel, T.~Koren, and Y.~Peres.
\newblock Bandit convex optimization: $\sqrt{T}$ regret in one dimension.
\newblock In {\em Proceedings of the 28th Annual Conference on Learning Theory
  (COLT)}, 2015.

\bibitem[CBFH{\etalchar{+}}97]{cesa1997use}
Nicolo Cesa-Bianchi, Yoav Freund, David Haussler, David~P Helmbold, Robert~E
  Schapire, and Manfred~K Warmuth.
\newblock How to use expert advice.
\newblock {\em Journal of the ACM (JACM)}, 44(3):427--485, 1997.

\bibitem[CL11]{empiricalTS}
O.~Chapelle and L.~Li.
\newblock {An Empirical Evaluation of Thompson Sampling}.
\newblock In {\em Advances in Neural Information Processing Systems (NIPS)},
  2011.

\bibitem[FK21]{foster2021efficient}
Dylan~J Foster and Akshay Krishnamurthy.
\newblock Efficient first-order contextual bandits: Prediction, allocation, and
  triangular discrimination.
\newblock {\em Advances in Neural Information Processing Systems}, 34, 2021.

\bibitem[Fre75]{freedman1975tail}
David~A Freedman.
\newblock On tail probabilities for martingales.
\newblock {\em The Annals of Probability}, 3(1):100--118, 1975.

\bibitem[KS12]{karatzas2012brownian}
Ioannis Karatzas and Steven Shreve.
\newblock {\em Brownian motion and stochastic calculus}, volume 113.
\newblock Springer Science \& Business Media, 2012.

\bibitem[KWK10]{KWK10}
W.~Koolen, M.~Warmuth, and J.~Kivinen.
\newblock Hedging structured concepts.
\newblock In {\em Proceedings of the 23rd Annual Conference on Learning Theory
  (COLT)}, 2010.

\bibitem[LG21]{lattimore2020mirror}
Tor Lattimore and Andras Gyorgy.
\newblock Mirror descent and the information ratio.
\newblock In {\em Conference on Learning Theory}, pages 2965--2992. PMLR, 2021.

\bibitem[LS19]{TShellinger}
Tor Lattimore and Csaba Szepesv{\'{a}}ri.
\newblock An information-theoretic approach to minimax regret in partial
  monitoring.
\newblock In {\em Conference on Learning Theory, {COLT} 2019, 25-28 June 2019,
  Phoenix, AZ, {USA}}, pages 2111--2139, 2019.

\bibitem[LST18]{LST18}
T.~Lykouris, K.~Sridharan, and E.~Tardos.
\newblock Small-loss bounds for online learning with partial information.
\newblock In {\em Proceedings of the 31st Annual Conference on Learning Theory
  (COLT)}, 2018.

\bibitem[LTW20]{thodoris2019graph}
Thodoris Lykouris, Eva Tardos, and Drishti Wali.
\newblock {Feedback graph regret bounds for Thompson Sampling and UCB}.
\newblock In {\em Proceedings of the 31st International Conference on
  Algorithmic Learning Theory (ALT)}, 2020.

\bibitem[LZS18]{liu2018analysis}
Fang Liu, Zizhan Zheng, and Ness Shroff.
\newblock {Analysis of Thompson Sampling for Graphical Bandits Without the
  Graphs}.
\newblock In {\em Proceedings of the 34th Conference on Uncertainty in
  Artificial Intelligence (UAI)}, 2018.

\bibitem[MS11]{mannor2011bandits}
Shie Mannor and Ohad Shamir.
\newblock From bandits to experts: On the value of side-observations.
\newblock In {\em Advances in Neural Information Processing Systems}, pages
  684--692, 2011.

\bibitem[RVR16]{RR14}
Daniel Russo and Benjamin Van~Roy.
\newblock An information-theoretic analysis of thompson sampling.
\newblock {\em The Journal of Machine Learning Research}, 17(1):2442--2471,
  2016.

\bibitem[TDD17]{tossou2017thompson}
Aristide~CY Tossou, Christos Dimitrakakis, and Devdatt~P Dubhashi.
\newblock Thompson sampling for stochastic bandits with graph feedback.
\newblock In {\em AAAI}, pages 2660--2666, 2017.

\bibitem[Tho33]{Tho33}
W.~Thompson.
\newblock On the likelihood that one unknown probability exceeds another in
  view of the evidence of two samples.
\newblock {\em Bulletin of the American Mathematics Society}, 25:285--294,
  1933.

\bibitem[ZL19]{TSandMD}
Julian Zimmert and Tor Lattimore.
\newblock {Connections Between Mirror Descent, Thompson Sampling and the
  Information Ratio}.
\newblock In {\em Advances in Neural Information Processing Systems 32 (NIPS)},
  2019.

\end{thebibliography}

\appendix

\section{Proof of Theorem~\ref{thm:bayesianexplorer}}

Here we prove Theorem~\ref{thm:bayesianexplorer}. Recall the statement:

\bayesianexplorer*

We recall the notations from Table~\ref{table:loss-notation}, which feature crucially in our proof.

\begin{center}
\begin{tabular}{ |c|c|c|c|  } 
 \hline
 $\ell_t^{\mathcal R}(i)=\ell_t(i)\cdot 1_{i\in \mathcal R_t}$ 
 & 
 $u_t^{\mathcal R}(i)=\frac{\ell^{\mathcal R}_t(i)\cdot 1_{i\in A_t}}{\gamma_2}$ 
 & 
 $L_t^{\mathcal R}(i)=\sum_{s\leq t}\ell_s^{\mathcal R}(i)$ 
 &  
 $U_t^{\mathcal R}(i)=\sum_{s\leq t}u_s^{\mathcal R}(i)$ \\ 
 \hline 
 $\ell_t^{\mathcal C}(i)=\ell_t(i)\cdot 1_{i\in \mathcal C_t}$ 
 & 
 $u_t^{\mathcal C}(i)=\frac{\ell^{\mathcal C}_t(i)\cdot 1_{i\in A_t}}{\hat p_t(i)}$ 
 & 
 $L_t^{\mathcal C}(i)=\sum_{s\leq t}\ell_s^{\mathcal C}(i)$ 
 & 
 $U_t^{\mathcal C}(i)=\sum_{s\leq t}u_s^{\mathcal C}(i)$  
 \\
 \hline
\end{tabular}
\end{center}

To control the error of the estimators $U_t$ we rely on Freedman's inequality (\cite{freedman1975tail}), a refinement of Hoeffding-Azuma which is more efficient for highly asymmetric summands.

\begin{theorem}[Freedman's Inequality]
\label{thm:freedman}
Let $S_t=\sum_{s\leq t} x_s$ be a martingale sequence, so that for some discrete-time filtration $(\mathcal F_t)_{t\in\mathbb Z_{\geq 0}}$,
\[
	\E[x_s|\mathcal F_{s-1}]= 0.
\] 
Suppose that a uniform and almost-sure one-sided estimate $x_s\leq M$ holds. Also define the conditional variance 
\[
	W_s=Var[X_s|\mathcal F_{s-1}]
\] 
and set $V_t=\sum_{s\leq t}W_s$ to be the total variance accumulated so far.

Then with probability at least $1-e^{-\frac{a^2}{2b+Ma}}$, we have $S_t\leq a$ for all $t$ with $V_t\leq b$.

\end{theorem}

Martingale concentration is useful to analyze the error of the unbiased estimators $U_t^{\mathcal C}(i)$. For the underbiased estimators it is correspondingly helpful to use \textbf{supermartingale} concentration. Recall that a supermartingale sequence $(S_t)_{t\geq 0}$ relative to a filtration $\mathcal F$ satisfies
\[
	\E[S_t~|~\mathcal F_{t-1}]\leq S_{t-1},
\]
i.e. it decreases on average.
Using a discrete-time Doob-Meyer decomposition (see e.g. \cite[Chapter 1.4]{karatzas2012brownian}) of a bounded supermartingale into the sum of a martingale and a decreasing predictable process, we obtain the following. (Here ``predictable'' means that $D_t$ is $\mathcal F_{t-1}$-measurable.)

\begin{corollary}
\label{cor:freedman}
Let $S_t=\sum_{s\leq t} x_s$ be a supermartingale sequence for $t\geq 1$, so that $\E[x_s|\mathcal F_{s-1}]\leq 0$. Suppose there is a uniform one-sided estimate $x_s-\E[x_s|\mathcal F_{s-1}]\leq M$. Also define the conditional variance 
\[
	W_s=Var[X_s|\mathcal F_{s-1}]
\] 
and set $V_t=\sum_{s\leq t}W_s$ to be the total variance accumulated so far.

Then with probability at least $1-e^{-\frac{a^2}{2b+Ma}}$, we have $S_t\leq a$ for all $t$ with $V_t\leq b$.
\end{corollary}

\begin{proof}
Write $S_t=M_t+D_t$ as the sum of a martingale $M_t$ and a decreasing predictable process $D_t$ with $D_1=0$. Explicitly, 
\begin{align*}
	M_t&=\sum_{1\leq s\leq t} S_s - \sum_{1\leq s\leq t-1}\E[S_{s+1}~|~\mathcal F_{s}];
	\\
	D_t&= \sum_{1\leq s\leq t-1} (S_s-\E[S_{s+1}~|~\mathcal F_{s}]).
\end{align*}
Then apply Theorem~\ref{thm:freedman} to $M_t$ and observe that $S_t\leq M_t$ almost surely for all $t$.
\end{proof}

Towards proving the two claims in Theorem~\ref{thm:bayesianexplorer} we first prove two lemmas. They follow directly from proper applications of Freedman's Theorem or its corollary. The second was used previously in the main body as well.

\begin{lemma}
\label{lem:bayesianexplorerR}
In the context of Theorem~\ref{thm:bayesianexplorer}, with probability at least $1-\frac{2}{T^2}$, for all $t$ with $L_t^{\mathcal R}(i)\leq \oL^*$ it holds that \[U_t^{\mathcal R}(i)\leq 2\oL^*+\frac{8\log T}{\gamma_2}.\]

\end{lemma}

\bayesianexplorerC*

\begin{remark}

Lemma~\ref{lem:bayesianexplorerC} has no dependence on $\oL^*$ and holds with $\oL^*=\infty$. For proving Theorem~\ref{thm:bayesianexplorer} we will simply take $\tilde L=\oL^*$. However it is necessary to apply Lemma~\ref{lem:bayesianexplorerC} with $\tilde L\neq \oL^*$ to analyze the semi-bandit setting.

\end{remark}

\begin{proof}{of Lemma~\ref{lem:bayesianexplorerR}:}\\ 

We analyze the (one-sided) error in the underestimate $U_t^{\mathcal R}(i)$ for $L_t^{\mathcal R}(i)$. Define the supermartingale $S_t=\sum_{s\leq t}x_s$ for 
\[
	x_s=x_s(i):=u_s^{\mathcal R}(i)-\ell_s^{\mathcal R}(i).
\]
We apply Corollary~\ref{cor:freedman} to this supermartingale, taking 
\[
	(a,b,M)=\left(\frac{4\log T}{\gamma_2}+4\sqrt{\frac{\oL^*\log T}{\gamma_2}},\frac{\oL^*}{\gamma_2}, \frac{1}{\gamma_2}\right).
\] 
For the filtration, we take the loss sequence $(\ell_t(i))_{t\in [T]}$ as known from the start so that the only randomness is from the player's choices. Equivalently, we act as the observing adversary; note that $S_t$ is still a supermartingale with respect to this filtration. Crucially, this means the conditional variance is bounded by $W_t\leq \frac{\ell_t^{\mathcal R}(i)}{\gamma_2}$. Therefore $V_t\leq \frac{L_t^{\mathcal R}(i)}{\gamma_2}$. Note also that with these parameters,
\[
	e^{-\frac{a^2}{2b+Ma}}\leq e^{-\frac{a^2}{4b}}+e^{-\frac{a}{2M}}\leq \frac{1}{T^2}+\frac{1}{T^2}=\frac{2}{T^2}.
\] 
Therefore by Freedman's inequality, with probability $1-\frac{2}{T^2}$, for all $t$ with $L_t^{\mathcal R}(i)\leq \oL^*$ we have 
\[
	S_t\leq a=\frac{4\log T}{\gamma_2}+4\sqrt{\frac{\oL^*\log T}{\gamma_2}}
\] 
and hence 
\begin{align*}
	U_t^{\mathcal R}(i)
	&\leq 
	L_t^{\mathcal R}(i)+\frac{4\log T}{\gamma_2}+4\sqrt{\frac{\oL^*\log T}{\gamma_2}}
	\\
	&\leq 
	\oL^*+\frac{4\log T}{\gamma_2}+4\sqrt{\frac{\oL^*\log T}{\gamma_2}}
	\\
	&\leq 
	2\oL^*+\frac{8\log T}{\gamma_2}.
\end{align*}
\end{proof}

\begin{proof}{of Lemma~\ref{lem:bayesianexplorerC}:}\\

As discussed previously we use the estimator
\[
	U_t^{\mathcal C}(i)=\sum_{s\leq t}  \frac{\ell^{\mathcal C}_s(i)\cdot1_{i_s=i}}{\hat p_s(i)}.
\]
for $L_t^{\mathcal C}(i)$.
We will again apply Freedman's inequality from the point of view of the adversary, this time to the martingale sequence $S_t=\sum_{s\leq t}x_s$ for 
\[
x_s=x_s(i):=\left(\frac{u_s^{\mathcal C}(i)}{\hat p_s(i)}-\ell_s^{\mathcal C}(i)\right).
\]

We have $x_s\leq \frac{1}{\gamma_1}=M$ and $V_t\leq \frac{L_t^{\mathcal C}(i)}{\gamma_1}$. We use the parameters $b=\frac{\tilde L}{\gamma_1}$ and $a=\lambda\sqrt{\frac{\tilde L}{\gamma_1}}$. Using $\gamma\geq\frac{1}{\tilde L}$ in the penultimate inequality and then $\lambda\geq 2$ yields the estimate:
\[
	e^{-\frac{a^2}{2b+Ma}}\leq e^{-\frac{a^2}{4b}}+e^{-\frac{a}{2M}}\leq e^{-\frac{\lambda^2}{4}}+e^{-\frac{\lambda^2\sqrt{\tilde L\gamma_1}}{2}}\leq e^{-\frac{\lambda^2}{4}}+e^{-\frac{\lambda}{2}}\leq 2e^{-\frac{\lambda}{2}}.
\]
Freedman's inequality implies that with probability at least $1-2e^{-\lambda/2}$, for all $t$ with $L_t^{\mathcal C}(i)\leq \tilde L$,
\[
	U_t^{\mathcal C}(i)\leq L_t^{\mathcal C}(i)+\lambda\sqrt{\frac{\tilde L}{\gamma_1}}.
\]
\end{proof}

Now we use these lemmas to prove Theorem~\ref{thm:bayesianexplorer}. In both halves, the main idea is that if something holds with high probability for any loss sequence, then the player must assign it high probability on average.

\begin{proof}{of Theorem~\ref{thm:bayesianexplorer}A:}\\

Let $E$ be the event that for all $t$ with $L_t^{\mathcal R}(i)\leq \oL^*$ we have \[U_t^{\mathcal R}(i)\leq 2\oL^*+\frac{8\log T}{\gamma_2}.\] By Lemma~\ref{lem:bayesianexplorerR}, $\mathbb P[E]\geq 1-\frac{2}{T^2}$ for any fixed loss sequence. The player does not know what the true loss sequence is, but his prior is a mixture of possible loss sequences, and so the player also assigns $E$ a probability at least $1-\frac{2}{T^2}$ at the start of the game. Let $F$ denote the event that
\[
	\mathbb P_t[E]\geq 1-\frac{1}{T},\quad \forall t\in [T].
\]
Since $\mathbb P_t[E]$ is a martingale, Doob's inequality implies
\[
	\mathbb P[F]\geq 1-\frac{2}{T}.
\] 

Assume now that $F$ holds, so that $\mathbb P_t[E]\geq 1-\frac{1}{T}$ at all times. Let $\tau$ be the first time at which 
\[
	U_{\tau}^{\mathcal R}(i)>2\oL^*+\frac{8\log T}{\gamma_2}.
\]
(If no such time exists, set $\tau=+\infty$.) Then as long as $E$ holds we must have $L_t^{\mathcal R}(i)>\oL^*$ and so $a^*\neq i$. Therefore, if $F$ holds then for all $t\geq\tau$,
\begin{align*}
	\mathbb P_t[i\in A_t]
	&=
	p_t(i)
	\\
	&=
	\mathbb P_t[i\in A^*]
	\\
	&\leq
	\mathbb P_t\left[L_t^{\mathcal R}(i)\leq \oL^*\right]
	\\
	&\leq
	1-\mathbb P_t[E]
	\\
	&\leq
	1/T.
\end{align*}
It follows that
\begin{equation}
\label{eq:F-true}
	1_{F}\cdot
	\sum_{t=\tau+1}^T
	p_t(i)
	\leq 
	1.
\end{equation}
On the other hand, since $\mathbb P[F]\geq 1-\frac{2}{T}$ the leftover contribution from $F$ being false is bounded by
\begin{equation}
\label{eq:F-false}
	\mathbb E
	\left[
		(1-1_{F})\cdot
		\sum_{t=\tau+1}^T
		p_t(i)
	\right]
	\leq 
	2.
\end{equation}

To finish, note that 
\[
	\gamma_2U_t^{\mathcal R}(i)
	=
	\sum_{s\leq t} \ell_s^{\mathcal R}(i)\cdot 1_{i\in A_s}
\]
is exactly the total loss paid by the player from arm $i$ while $i\in\mathcal R_t$ is rare. Therefore $\tau$ is the smallest value satisfying 
\[
	\gamma_2 U_{\tau}^{\mathcal R}(i)
	> 
	\gamma_2\left(2\oL^*+\frac{8\log T}{\gamma_2}\right)=2\gamma_2 \oL^*+8\log T.
\] 
Since the increments of $U_t^{\mathcal R}(i)$ are bounded by $1/\gamma_2$, we have almost surely 
\begin{align*}
	\sum_{t\leq \tau} \ell_t^{\mathcal R}(i)1_{i\in A_t}
	&=
	\gamma_2 U_{\tau}^{\mathcal R}(i)
	\\
	&\leq 
	2\gamma_2 \oL^*+8\log(T)+1.
\end{align*}
Combining with \eqref{eq:F-true} and \eqref{eq:F-false} we finally obtain
\begin{align*}
	\E\left[
		\sum_{t\in [T]:\text{ }i\in \mathcal R_t} \hat p_t(i)\ell_t(i)
	\right]
	&=
	\E\left[
		\sum_{t\in [T]:\text{ }i\in \mathcal R_t}\ell_t(i)1_{i\in A_t}
	\right]
	\\
	&=
	\E\left[
		\sum_{t\in [T]}\ell_t^{\mathcal R}(i)1_{i\in A_t}
	\right]
	\\
	&\leq 
	\E\left[
		\sum_{t\leq \tau}\ell_t^{\mathcal R}(i)1_{i\in A_t}
	\right]
	+
	\E\left[
		1_F \sum_{t=\tau+1}^T\ell_t^{\mathcal R}(i)1_{i\in A_t}
	\right]
	\\
	&\quad\quad\quad
	+
	\E\left[
		(1-1_F) \sum_{t=\tau+1}^T\ell_t^{\mathcal R}(i)1_{i\in A_t}
	\right]
	\\
	&\leq
	2\gamma_2 \oL^*+8\log T + 4.
\end{align*}
\end{proof}

\begin{proof}{of Theorem~\ref{thm:bayesianexplorer}B:}\\

For $\lambda\geq 0$, let $E_{\lambda}$ be the event that for all $t$ with $L_t^{\mathcal C}(i)\leq \oL^*$,
\[
	U_t^{\mathcal C}(i)\leq L_t^{\mathcal C}(i)+\lambda\sqrt{\frac{\oL^*}{\gamma_1}}.
\]
We apply Lemma~\ref{lem:bayesianexplorerC} with $\tilde L=\oL^*$, obtaining 
\[
	\mathbb P[E_{\lambda}]\geq 1-2e^{-\lambda/2},
	\quad\quad
	\forall\lambda>2.
\]
Let $\tau_{\lambda}$ be the first time such that 
\[
	U_{\tau_{\lambda}}^{\mathcal C}(i)
	>
	\oL^*+\lambda\sqrt{\frac{\oL^*}{\gamma_1}}.
\] 
(If no such time exists, take $\tau_{\lambda}=+\infty$.) As before, note that at the start we have
\[
	\mathbb P_1[E]\geq 1-2e^{-\lambda/2}
\]
since the initial prior is some mixture of loss sequences. By definition, if $E_{\lambda}$ holds and $\tau_{\lambda}<\infty$ then $i\notin A^*$. Hence 
\begin{align*}
	\E[p_{\tau_{\lambda}\wedge T}(i)]
	&\leq
	\E[1-\mathbb P_{\tau_{\lambda}}[E_{\lambda}]]
	\\
	&=
	1-\mathbb P[E_{\lambda}]
	\\
	&\leq 
	2e^{-\lambda/2}
\end{align*}
by optional stopping (on the martingale $p_t(i)$) since $U_t^{\mathcal C}(i)$ is computable by the player (i.e. adapted to the player's filtration). By Doob's inequality applied to the same martingale,
\begin{align*}
	\mathbb P\left[\sup_{t\in[\tau_{\lambda},T]} p_t(i) >\gamma_1\right]
	&\leq
	\frac{\E[p_{\tau_{\lambda}\wedge T}(i)]}{\gamma_1}
	\\
	&\leq
	\frac{2e^{-\lambda/2}}{\gamma_1}
	\\
	&= 
	2e^{-\frac{\lambda-2\log(1/\gamma_1)}{2}}.
\end{align*}
Now, let $\lambda^*$ be such that $U_t^{\mathcal C}(i)=\oL^*+\lambda^*\sqrt{\frac{\oL^*}{\gamma_1}}$ at the last time $t$ when $p_t(i)>\gamma_1$.
What we have just shown is equivalent to 
\[
	\mathbb P[\lambda^*>\lambda]\leq 2e^{-\frac{\lambda-2\log(1/\gamma_1)}{2}}.
\] 
In other words, $\lambda^*$ has tail bounded above by an exponential random variable with half-life $2\log(2)$ starting at $2\log(1/\gamma_1)+2\log(2)$, and therefore \[\E[\lambda^*]\leq 2\log(1/\gamma_1)+10.\] However, we always have $U_T^{\mathcal C}(i)=\oL^*+\lambda^*\sqrt{\frac{\oL^*}{\gamma_1}}$ since after the last time $t$ with $p_t(i)>\gamma_1$, the value of $U_t^{\mathcal C}(i)$ cannot change. Recall also that $U_T^{\mathcal C}(i)$ is an unbiased estimator for $L_T^{\mathcal C}(i)$. Combining completes the proof:
\begin{align*}
	\E[L_T^{\mathcal C}(i)]
	&=\E[U_T^{\mathcal C}(i)]
	\\
	&= \oL^*+\E[\lambda^*]\sqrt{\frac{\oL^*}{\gamma_1}}
	\\
	&\leq \oL^*+2\left(\log\left(\frac{1}{\gamma_1}\right)+10\right)\sqrt{\frac{\oL^*}{\gamma_1}}.
\end{align*}
\end{proof}

\end{document}